\documentclass{article}

\PassOptionsToPackage{numbers, compress}{natbib}
\usepackage[preprint]{neurips_2024}

\usepackage[utf8]{inputenc} % allow utf-8 input
\usepackage[T1]{fontenc}    % use 8-bit T1 fonts
\usepackage{url}            % simple URL typesetting

\usepackage{booktabs}       % professional-quality tables
\usepackage{amsfonts}       % blackboard math symbols
\usepackage{nicefrac}       % compact symbols for 1/2, etc.

\usepackage{microtype}      % microtypography
\usepackage{xcolor}         % colors
\usepackage{graphicx}
\usepackage{algorithm}
\usepackage{multirow}
\usepackage{algorithmic}
\usepackage{subcaption}
\usepackage{forloop}
\usepackage{ifthen}
\usepackage{enumitem}
\usepackage{soul}
\usepackage{amsmath}

\newcommand{\para}[1]{\noindent{\textbf{{#1}}}}
\usepackage{amssymb}
\newcommand{\ie}{{\textit{i.e.}}}
\usepackage{wrapfig}
\usepackage{adjustbox}

\usepackage[pagebackref,breaklinks,colorlinks,citecolor=blue]{hyperref}
\usepackage{amsthm}
\usepackage{mathrsfs}
\newtheorem{definition}{Definition}
\title{E$^3$-Net: Efficient E(3)-Equivariant Normal Estimation Network}

\author{%
   Hanxiao Wang\\
   CASIA\\
  \And
  Mingyang Zhao\\
  HKISI\\
  \And
  Weize Quan\\
  CASIA\\
  \And
  Zhen Chen\\
  HKISI\\
  \And
  Dong-ming Yan\\
  CASIA\\
  \And
  Peter Wonka\\
  KAUST\\
}

\begin{document}

\maketitle

\begin{abstract}
Point cloud normal estimation is a fundamental task in 3D geometry processing. While recent learning-based methods achieve notable advancements in normal prediction, they often overlook the critical aspect of \emph{equivariance}. This results in inefficient learning of symmetric patterns. To address this issue, we propose E$^3$-Net to achieve equivariance for normal estimation. We introduce an efficient random frame method, which significantly reduces the training resources required for this task to just 1/8 of previous work and improves the accuracy. Further, we design a Gaussian-weighted loss function and a receptive-aware inference strategy that effectively utilizes the local properties of point clouds. Our method achieves superior results on both synthetic and real-world datasets, and outperforms current state-of-the-art techniques by a substantial margin.
We improve RMSE by 4\% on the PCPNet dataset, 2.67\% on the SceneNN dataset, and 2.44\% on the FamousShape dataset.
% particularly in scenarios involving complex geometry or non-uniformly sampled data, where it achieves an improvement of over 10\%. 
% Extensive ablation studies further validate the efficacy of each component of our method in handling challenging point clouds.

\end{abstract}

%   \includegraphics[width=\textwidth]{Fig/teaser4.pdf}
%   %\vskip -0.3cm
%   \caption{We propose E$^3$-Net, an efficient {E}(3)-equivariant normal estimation network for point cloud analysis. Top: Normal estimation results and the corresponding angle root mean squared error. Bottom: The Poisson reconstruction results and the Hausdorff distance to the ground truth surface {(scaled by 
%  ${10}^{-4} \times$ the diagonal length of the bounding box)}.}
%   % \Description{We propose a  Efficient E(3)-Equivariant Normal estimation method.Comparative Analysis of Normal estimation with the Noisy Liberty Model. Top row:   the results of the  Normal estimation. Bottom row: the reconstruction results, accompanied by the corresponding RMSE (Root Mean Square Error) values.}

\section{Introduction}
Normal estimation for point cloud analysis is a fundamental and essential task in 3D vision. It involves accurately determining the surface normals for each point and thus plays a crucial role in numerous applications, \emph{e.g.}, 3D reconstruction~\citep{kazhdan2006poisson, kazhdan2013screened}, object recognition~\citep{drost2010model,brostow2008segmentation}, denoising~\citep{Fleishman_2003, Avron_2010}, ... 
% Despite significant advancements in this area, existing approaches still face challenges when dealing with point clouds that exhibit characteristics such as noise, non-uniform density, or complex geometry.% (Fig.~\ref{fig:teaser}).

Traditional methods for normal estimation include, \emph{Principal Component Analysis} (PCA)~\citep{hoppe1992surfacepca}, n-jet~\citep{cazals2005estimatingnjet}, \emph{Moving Least Squares} (MLS)~\citep{levin1998approximationmls,mitra2003estimating}, Voronoi-based method~\citep{merigot2010voronoi}, Hough transform~\citep{boulch2012fast} and winding-number field~\citep{xu2023globally}.
% , and Local Surface Fitting~\citep{weiss2002advanced}. 
While these classical methods have laid the groundwork for point cloud processing, their limitations in handling noise, computational efficiency, adaptability to complex structures, and dependency on data density necessitate further advancements in the field.

Recent neural network-based algorithms for estimating normals in point clouds can overcome many of the challenges of traditional methods. PCPNet~\citep{guerrero2018pcpnet} proposed a single-patch-based framework, which was followed by n-jet fitting methods~\citep{ben2020deepfit,zhu2021adafit,li2022graphfit} and implicit surface methods~\citep{li2022hsurf,li2023shs}. Recently, MSECNet~\citep{xiu2023msecnet} introduced a patch-to-patch approach, further improving efficiency and accuracy. 
Additionally, compared with traditional approaches like PCA and MLS, learning-based methods lack invariant properties that are independent of the point cloud's embedding in three-dimensional (3D) space, particularly those relying on GNN~\citep{Scarselli_2009gnn} or PointNet~\citep{Charles_2017pointnet} architectures. As a result, their results may vary with changes in orientation or position, as these methods do not inherently guarantee such properties. Equivariant networks like frame averaging~\citep{Puny_2021frameaverage} and E($n$)-GNN~\citep{satorras2021nEngnn} can mitigate this issue but often require significantly increased inference time or parameter count for training.

To address the aforementioned issues, building on the analysis of E(3)-equivariance presented in Puny \textit{et al.}~\cite{Puny_2021frameaverage}, we develop a novel method termed \emph{E$^3$-Net}, an \emph{Efficient E(3)-Equivariant} network for normal estimation that enables both \emph{equivariance} and \emph{efficiency}.
Instead of employing a large network that processes multiple frames jointly, we utilize a smaller, more efficient network that processes one frame at a time. This approach requires only 1/8 of the training resources.

Furthermore, we tackle inefficiencies due to the number of patches being processed.  Processing a single patch for each point, may have higher accuracy, but is highly inefficient. In contrast, predicting normals for all points within a patch significantly increases efficiency but introduces issues at boundaries, as boundary points have a smaller receptive field. To address these challenges, we develop a new loss function that incorporates Gaussian weighting, which enables simultaneous prediction of normals for all points in an entire patch.

To increase the receptive field during the testing phase, we introduce a novel receptive-aware inference strategy that supplants the traditional k-nearest neighbor method with geodesic patch sampling. Additionally, we refine our aggregation strategy to be receptive-aware, drawing inspiration from mean filtering and Gaussian filtering, as reviewed by Buades et al.~\citep{buades2005review}. This method effectively harnesses information from overlapping patches, thus significantly improving the performance of our model.

We conduct extensive comparative experiments to demonstrate the efficacy of the proposed method on typical benchmark datasets, including PCPNet~\citep{guerrero2018pcpnet}, FamousShape~\citep{li2023shs}, and real-world SceneNN~\citep{li2022hsurf}. Our method consistently outperforms current state-of-the-art models.
% , showcasing an improvement up to 10\% with noiseless or non-uniformly sampled data. 
Furthermore, we validate the practical application potential of our algorithm by utilizing the predicted normals for 3D Poission reconstruction~\citep{kazhdan2006poisson}. To summarize, the main contributions of this work are as follows:
\begin{itemize}[leftmargin=*]
\item a novel method for random frame training which ensures both efficiency and equivariance in the network (See Sec.~\ref{subsec:E3Equivarance}).
%\item We introduce a training strategy with random frames and an inference strategy with average frames, endowing the network with E3 equivariance.
%{\item An inference strategy that includes extracting geodesic patches and employing proximity-based selection, combined with Gaussian weighting. This approach effectively leverages the local properties of point clouds, resulting in enhanced aggregation capabilities.

\item a novel Gaussian-weighted loss function, for efficient patch-based processing without loss of accuracy (See Sec.~\ref{subsec:LossFunction}).

\item a receptive-aware inference strategy using geodesic patches and weighted aggregation (See Sec.~\ref{subsec:Inference}).

% \item Our method achieves state-of-the-art results on diverse datasets and surpasses competing approaches by a significant margin.

\end{itemize}

\section{Related Work}
We present an overview of traditional normal estimation,  learning-based normal estimation, and E($n$)-equivariant networks.

\para{Traditional Normal Estimation.} In point cloud normal estimation, traditional methods typically involve analyzing a small neighborhood around each point. PCA~\citep{hoppe1992surfacepca} examined local covariance within a small area to determine the direction of minimal variance as the normal. MLS~\citep{levin1998approximationmls} fits a plane to local points for normal estimation. n-Jets~\citep{cazals2005estimatingnjet} estimated normals by fitting higher-order polynomial surfaces to point cloud data, while Merigot \textit{et al.}~\cite{merigot2010voronoi} leveraged the covariance matrices of Voronoi cells to predict normals.
The Hough Transform method~\citep{boulch2012fast} interpreted the filled accumulator as a discrete probability distribution, where the most probable normal direction was determined by the maximum value in this distribution. These traditional methods are generally sensitive to noisy data and do not perform as well as learning-based alternatives.

\para{Learning-based Normal Estimation.}
Early work, PCPNet~\citep{guerrero2018pcpnet} proposed a patch-based normal prediction framework. Zhou \textit{et al.}~\cite{zhou2020normal} employed local plane constraints and multi-scale selection in their approach. MTRNet~\citep{cao2021latent} focused on learning latent tangent spaces combined with a compact multi-scale approach. Furthermore, Nesti-Net~\citep{ben2019nesti} integrated multi-scale analysis with 3D convolutional neural networks. Lessen \textit{et al.}~\cite{lenssen2020deep} proposed a method involving iterative surface fitting, neural network-driven inference, and adaptive kernel refinement. In addition, Refine-Net~\citep{zhou2022refine} integrated advanced feature representations with a novel connection module within a geometrically-focused framework. Zhang \textit{et al.}~\cite{zhang2022geometry} introduced a geometry-weighted guidance approach, combined with multi-scale normal integration and error correction networks. In jet-based approaches, DeepFit~\citep{ben2020deepfit} combined PointNet-based weight prediction with jet fitting to enhance the robustness and data efficiency in normal estimation. AdaFit~\citep{zhu2021adafit} introduced an offset prediction mechanism in jet fitting to adaptively handle complex surface geometries, while Graphfit~\citep{li2022graphfit} combined graph-convolutional learning, adaptive attention-based fusion, and multi-scale representation to further enhance normal estimation accuracy. Du \textit{et al.}~\cite{du2023rethinking} focused on $z$-direction transformation and normal error estimation, which can be flexibly integrated with current polynomial surface fitting methods for point cloud normal estimation. Additionally, NeAF~\citep{li2023neaf} innovatively learned an angle field around point normals, predicting offsets between input vectors and true normals. HSurf-Net~\citep{li2022hsurf} combined hyper surface fitting in a high-dimensional feature space without the need for adjusting the order. Recent developments like SHS-Net~\citep{li2023shs} and NGL~\citep{li2023neural} focused on oriented point normal estimation, typically using a one-patch-per-point prediction approach, which is less efficient. More recently, MSECNet~\citep{xiu2023msecnet} adopted a patch-to-patch prediction method, enhancing efficiency over other methods. Different from previous approaches, our work features all three important attributes: equivariance, efficiency, and accuracy (beating all previous methods).

\begin{wrapfigure}{r}{0.5\textwidth} %
    \centering
     % \vskip -1 cm
    \includegraphics[width=0.35\textwidth]{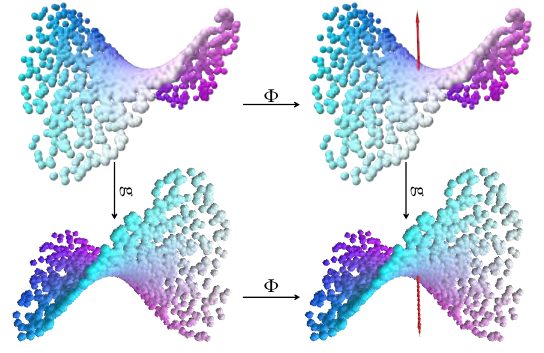} 
    \caption{An illustration of the E(3)-equivariance, where $\Phi$ is an E(3)-equivariant function predicting a point's normal, and \(g\) is an element of the E(3) group. These two operators are commutative, meaning that the order in which they are applied does not affect the final outcome.}
    \label{fig:enter-label1}
    % \vskip -1cm
\end{wrapfigure}
\para{E($n$)-Equivariant Networks}
Significant advancements have been achieved in neural networks regarding symmetry invariance and equivariance~\citep{deng2021vector, li2021rotation, chen2021equivariant, assaad2022vn, esteves2018learning}, expanding applications across various fields. SchNet~\citep{schutt2018schnet}
% , a continuous-filter convolutional network, 
exemplified this progress by modeling quantum interactions with rotationally invariant energies and equivariant forces. Tensor field networks~\citep{thomas2018tensor} introduced rotation, translation, and permutation equivariance at each layer, enhancing feature identification. Equivariant flows~\citep{kohler2020equivariant} innovated in symmetric energy sampling for multi-body systems, guiding the design of symmetry invariant Boltzmann generators. E($n$)-equivariant graph neural networks~\citep{satorras2021nEngnn} demonstrated its equivariance to rotations, translations, reflections, and permutations, without relying on higher-order representations. Additionally, SE(3)-transformers~\citep{fuchs2020se} marked a leap with their 3D roto-translationally equivariant self-attention module. Frame averaging \cite{Puny_2021frameaverage} emerged as a highly efficient and robust approach, adapting neural architectures to various symmetries, which led to applications in equivariant shape space learning~\citep{atzmon2022frame}. E3Sym~\citep{li2023e3sym} showcased E(3) invariance in unsupervised 3D planar reflective symmetry detection. FAENet~\citep{duval2023faenet} highlighted the use of frame averaging equivariant GNN in materials modeling, proving the effectiveness of E(3)-equivariant GNNs.
% Taking inspiration from these advancements, we incorporate the concept of equivariance into normal estimation.

\section{Preliminaries}

\begin{definition}
 The Euclidean group E($n$) that encompasses all isometric transformations  including rotations, reflections, and translations within $n$ dimension is defined as: 
\begin{equation}
    {\mathrm{E}(n)} = {\mathrm{O}(n)} \ltimes \mathbb{R}^n \triangleq \{(\mathbf{R}, \mathbf{v}) \mid \mathbf{R} \in \mathrm{O}(n), \mathbf{v} \in \mathbb{R}^n\},
\end{equation}
where $\ltimes$ is the semidirect product operator and O($n$) signifies the set of $n \times n$ orthogonal matrices expressed as:
\begin{equation}
{\mathrm{O}(n)} = \{ \mathbf{R} \in \mathbb{R}^{n\times n} \mid \mathbf{R} \mathbf{R}^T = \mathbf{R}^T \mathbf{R} = \mathbf{I}, \det(\mathbf{R}) = \pm 1 \}.
\end{equation}   
\end{definition}

\begin{definition}
A function \(\Phi\) is considered to be \emph{E($n$)-equivariant} if, when an input \(\mathbf{x}\) is transformed by an element $g\in\mathrm{E}(n)$, the resulting output of the function is likewise transformed by \(g\). Formally, this is expressed as:
\begin{equation}
    \Phi(g(\mathbf{x})) = g(\Phi(\mathbf{x})), \ \text{for all} \ g \in {\mathrm{E}(n)}.
     \label{eq:e3}
\end{equation}  
\end{definition}

This concept plays a crucial role in normal estimation. As shown in Fig.~\ref{fig:enter-label1}, equivariance ensures that estimating normals using a local patch and then applying an isometric transformation to the estimated normals is equivalent to first applying an isometric transformation to the local patch and then estimating the normals. The outcome remains consistent regardless of the order in which isometric transformation and normal estimation are applied.

\para{Frame Construction.}
Given an input point patch $\mathbf{x}\in\mathbb{R}^{m\times 3}$, we first calculate its centroid denoted as \(\boldsymbol{t}\), and then we compute its covariance matrix. Assume that the eigenvalues of this covariance matrix are ordered in ascending order,  \emph{e.g.}, \(\lambda_1\leq\lambda_2\leq\lambda_3\). Corresponding to these eigenvalues, we have normalized eigenvectors denoted by \(\boldsymbol{v}_1, \boldsymbol{v}_2,\boldsymbol{v}_3\). Then, the frame set that relates $\boldsymbol{v}_i$ and $\boldsymbol{t}$ belonging to $\mathrm{E}(3)$ is constructed as~\citep{Puny_2021frameaverage}
\begin{equation}
    \mathcal{F}(\mathbf{x}) = \left\{ \left( [\pm \boldsymbol{v}_1, \pm\boldsymbol{v}_2, \pm \boldsymbol{v}_3], \boldsymbol{t} \right)  \right\} \subset \mathrm{E}(3).
    \label{eq:e3}
\end{equation}
Note that the frame set has \(2^3\) elements. 
Compared to the uncountable E(3), the frame set constructed in Eq.~\ref{eq:e3} contains only a finite number of elements (8), yet it includes all the information necessary for the algorithm to possess E(3)-equivariance, which offers a systematic approach for algorithm design.

\begin{figure*}[t]
    \centering
    % \vskip -0.3cm
    \includegraphics[width=1\linewidth]{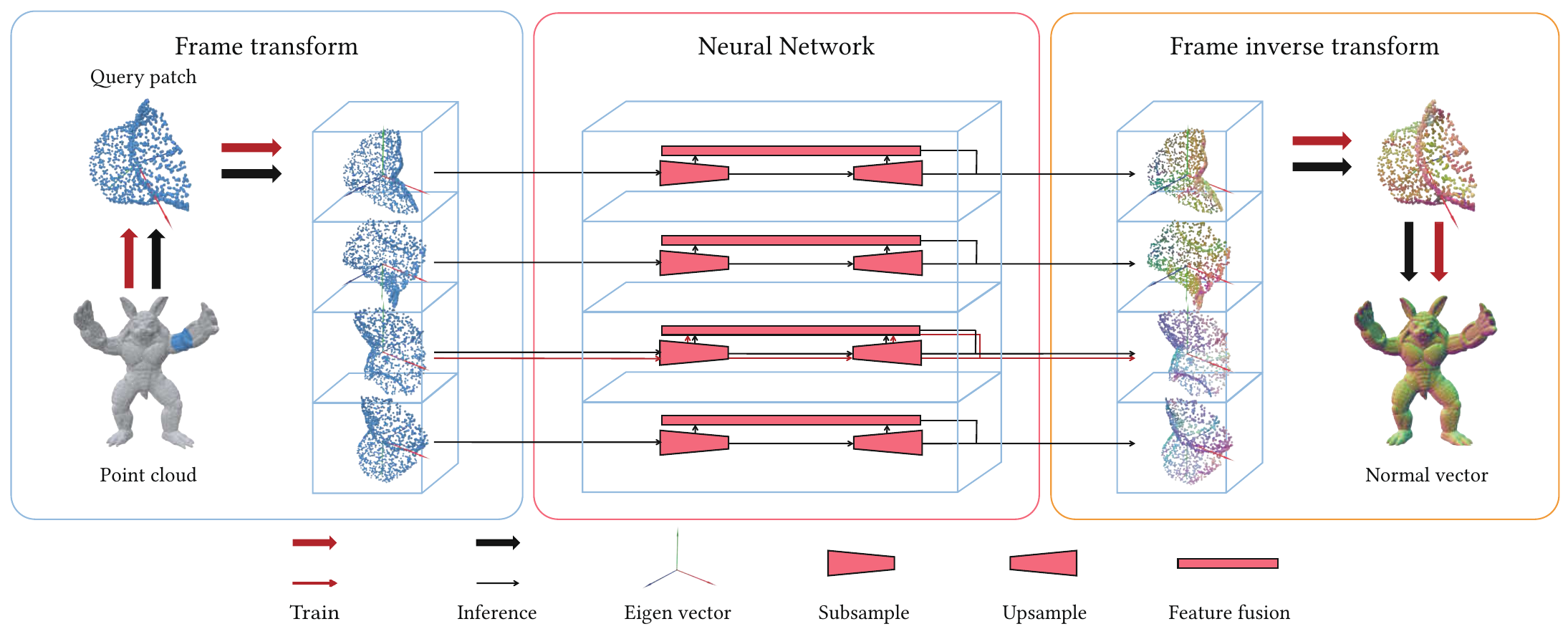}
    % \vskip -0.3cm
    \caption{Illustration of the proposed framework. The red arrows depict the training process, while the black arrows represent the inference phase. The diagram only shows 4 elements in the frame set.}

    \label{fig:main}
    \vskip -0.3cm
\end{figure*}

\section{Methodology}
In this section, we present our main contributions: efficient E(3)-equivariance (Sec.~\ref{subsec:E3Equivarance}), a Gaussian weighted loss function (Sec.~\ref{subsec:LossFunction}), and receptive-aware inference (Sec.~\ref{subsec:Inference}).

\subsection{Efficient  E(3)-Equivariance}
\label{subsec:E3Equivarance}
To ensure the E(3)-equivariant property, we adopt the frame averaging method defined by the following equation,  
\begin{equation}
\Phi(\mathbf{x}) = \frac{1}{|\mathcal{F}(\mathbf{x})|} \sum_{g \in \mathcal{F}(\mathbf{x})} g \phi(g^{-1}\mathbf{x}),
\label{eq:frame_ave}
\end{equation}

This approach utilizes $\mathcal{F}(\mathbf{x})$, a small subset of the group for averaging, thus making the process more practical for applications involving the Euclidean group E(3). In our context, $\Phi$ represents the E(3)-equivariant normal estimation network.

Although Eq.~\ref{eq:frame_ave} guarantees 
E(3)-equivariance (theoretical proof is presented in the \emph{Appendix}~\ref{appendix:proof}), the training of a network that involves all frames simultaneously, typically requires substantial computational resources and is difficult to optimize. To boost the efficiency, we further propose a random frame strategy as illustrated in Fig.~\ref{fig:main}. During the training stage, a random frame \({\mathcal{F}(\mathbf{x})}_i\) is first chosen from \({\mathcal{F}(\mathbf{x})}\). Then the data is transformed to this frame before being fed into the network. Our random scheme not only assures E(3)-equivariance, but also allows \(\phi\) to handle data in various frames more efficiently. During inference, this transformation can be directly applied to implement \(\Phi\), facilitating the processing of data in its original frame context, thereby achieving E(3)-equivariant property without the need of training the entire $\Phi$.
In summary, previous work~\cite{Puny_2021frameaverage} focused on $\Phi$ on the left side of the equation computing all eight frames jointly, but we focus $\phi$ on the right side computing one $\phi$ at a time. This leads to a factor 8 improvement in efficiency considering the size of the network, while at the same time improving accuracy.

\subsection{Loss Function}
\label{subsec:LossFunction}

In our method, two loss functions are utilized to estimate normals. The first function is defined as a regression loss, which minimizes the Euclidean distance between the predicted normals and the ground truth:
\begin{equation}
    \mathcal{L}^{reg}(\mathbf{n}_i) =  \min \left( \left\| \hat{\mathbf{n}}_i - \mathbf{n}_i \right\|_2^2, \left\| \hat{\mathbf{n}}_i + \mathbf{n}_i \right\|_2^2 \right),
\end{equation}
where $\hat{\mathbf{n}}_i$ and ${\mathbf{n}}_i$ denote the predicted and the ground truth normals for the $i$-th point, respectively.

The second function is the sine loss, focusing on the angular deviation between the predicted normals and the ground truth. It is defined as:
% \begin{equation}
%     \mathcal{L}^{\sin}(\mathbf{n}) = \frac{1}{N} \sum_{i=1}^N \left| \widetilde{\mathbf{n}}_i \times \mathbf{n}_i \right|,
% \end{equation}
\begin{equation}
    \mathcal{L}^{\sin}(\mathbf{n}_i) =  \left\| \hat{\mathbf{n}}_i \times \mathbf{n}_i \right\|_2,
\end{equation}
with $\times$ indicating the cross product. Therefore, the cumulative loss is a blend of these two components:
\begin{equation}
    \mathcal{L}^{Val}(\mathbf{n}) = \frac{1}{N} \sum_{i=1}^N (\mathcal{L}^{reg}(\mathbf{n}_i) + \mathcal{L}^{\sin}(\mathbf{n_i})).
\end{equation}

\iffalse
\begin{wrapfigure}{r}{0.5\textwidth} %
    \centering
    \vskip -0.6cm
    \includegraphics[width=0.9\linewidth]{Fig/loss.pdf}
    \vskip -0.1cm
    \caption{The proposed distance weighting scheme for normal estimation. The length of the red arrows indicates the loss weight for the point. From the left to right are the point-based method, \(\mathcal{L}^{\text{Val}}\), and \(\mathcal{L}^{\text{Gau}}\). 
 }   
    %Loss optimization target diagram, with the length of red arrows indicating the loss weight for the point. From left to right are the point based method, \(\mathcal{L}^{\text{Val}}\), and \(\mathcal{L}^{\text{Gau}}\).

    \label{fig:enter-label3}
    \vskip -0.5cm
\end{wrapfigure}
\fi
In previous normal estimation algorithms~\citep{ben2020deepfit,li2022graphfit,zhu2021adafit}, a single patch is typically used to predict the normal of a single point. These approaches aim to provide more detailed information about that specific point. In contrast, MSECNet~\citep{xiu2023msecnet} utilized a patch to predict normals for all points inside the patch. 
% However, this can be challenging for points on the boundary due to insufficient neighborhood information. To address this issue, especially for points near boundaries, we employ Gaussian distance weighting. This weighting scheme increases penalties for points near the center and decreases them for boundary points, formulated as:
However, this can be challenging for points on the boundary due to insufficient neighborhood information. The receptive field of boundary points is inherently limited, containing information from only one side, which complicates obtaining accurate predictions. To address this issue, especially for points near boundaries, we employ Gaussian distance weighting. This weighting scheme increases penalties for points near the center and decreases them for boundary points:

\begin{equation}
    \mathcal{L}^{{Gau}}(\mathbf{n}) = \frac{1}{N} \sum_{i=1}^N w_i \cdot (\mathcal{L}^{reg}_i + \mathcal{L}^{\sin}_i),
\end{equation}
where $w_i = \exp\left(-\frac{d_i^2}{2\sigma^2}\right)$ is the Gaussian weight for the $i$-th point, $d_i$ is the distance of the $i$-th point from the patch center, and $\sigma$ is the standard deviation of the Gaussian distribution. As shown in Fig.~\ref{fig:enter-label3}, our method optimizes the normals across the entire patch while simultaneously focusing on points with more comprehensive neighborhood information. This scheme
not only improves learning efficiency but also
enhances the estimation accuracy.

\begin{figure}[t]
    \centerline{
    \begin{minipage}{0.48\textwidth}
        \centering
\includegraphics[width=1.0\linewidth]{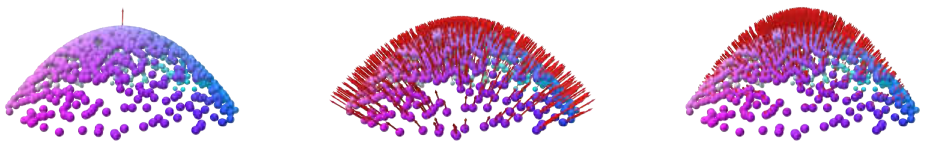}
        % \vskip -0.3cm
        \caption{The proposed distance weighting scheme for normal estimation. The length of the red arrows indicates the loss weight for the point. From the left to right are the point-based method, \(\mathcal{L}^{\text{Val}}\), and \(\mathcal{L}^{\text{Gau}}\). }
        \label{fig:enter-label3}
    \end{minipage}
    
    \hfill
    
    \begin{minipage}{0.48\textwidth}
    \centering
    % \vskip 0.3cm
    \includegraphics[width=1.0\linewidth]{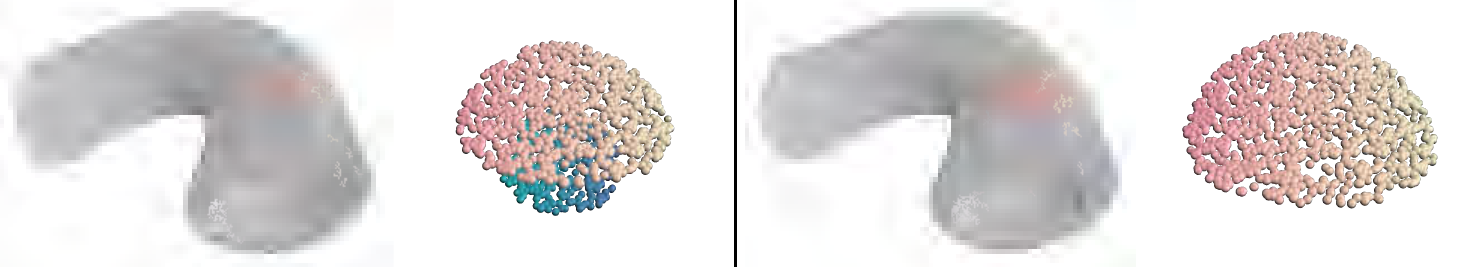}
     \vskip -0.1cm
    \caption{Difference between Euclidean (left) and geodesic patches (right). Euclidean patches contain points from both sides, whereas geodesic patches only have points from one side.}%{(scaled as the definition in Fig.~\ref{fig:teaser})}.}
%The meanings of the numbers in the figure are the same as in the second line of Fig.~\ref{fig:teaser}}
    \label{fig:patch}
    % \vskip -0.5cm
    \end{minipage}
    }
\end{figure}

\subsection{Inference}
\label{subsec:Inference}

\iffalse
\begin{wrapfigure}{r}{0.5\textwidth} %
    \centering
\vskip -0.5cm
\includegraphics[width=0.9\linewidth]{Fig/patch3.pdf}
\vskip -0.3cm
    \caption{Difference between Euclidean (left) and geodesic patches (right). Euclidean patches contain points from both sides, whereas geodesic patches only have points from one side.
%only takes the neighborhood of the point.
}
    \label{fig:patch}
    \vskip -0.5cm
\end{wrapfigure}
\fi
% \begin{figure}[t]
%     \centering
% \vskip -0.2cm
% \includegraphics[width=0.9\linewidth]{Fig/patch3.pdf}
% \vskip -0.3cm
%     \caption{Difference between kNN (left) and geodesic patches (right). kNN takes points on both sides as the same patch, whereas geodesic ensures the shortest route on a surface.
% %only takes the neighborhood of the point.
% }
%     \label{fig:patch}
%     \vskip -0.5cm
% \end{figure}
Previous algorithms have commonly relied on Euclidean $k$-nearest neighbors (kNN) graphs to define patches within point clouds during inference. However, kNN, by its inherent nature, disregards \emph{geodesic} distances and may unintentionally group closeby points from different surfaces into the same patch. As shown in Fig.~\ref{fig:patch}, this can result in inaccuracies and ambiguity, particularly when dealing with complex geometric structures. To overcome this issue, we employ Dijkstra's algorithm in conjunction with a graph constructed from the point cloud to provide an effective approximation for geodesic distances. We generate patches by determining the nearest points through Dijkstra's algorithm on this graph. When the connected components of the point cloud contain fewer points than the desired patch size, our algorithm reverts to the kNN approach for patch construction, which ensures each patch is adequately populated, maintaining robustness even for sparse data.
%in scenarios with sparse data.

After the construction of point cloud patches, we compute \(\phi\) over the entire \({\mathcal{F}(\mathbf{x})}\) of the patches. By averaging these results, we achieve E(3)-equivariant outcomes \(\Phi\). Additionally, we combine proximity-based selection and Gaussian weighting to handle overlapping areas that typically result in multiple normal predictions for each point, \ie,  points nearer to a patch's center are weighted more, reflecting their larger receptive field and points that are excessively distant from the center are discarded. We provide a pseudocode summary in \emph{Appendix~\ref{appendix:pseudo}} to make the above inference process more understandable.

\section{Experimental Results}
\para{Implementation.} 
During training, we set the patch size to $1400$.
We use MSECNet~\citep{xiu2023msecnet} as our backbone, with
patch sizes of 16 and 9 as hyperparameters of the kNN algorithm in the down sampling and up sampling block in the network, and 1024 as the feature fusion dimension. We employ the AdamW~\citep{loshchilov2018decoupled} optimizer for optimization, starting with an initial learning rate of \(2 \times 10^{-3}\) and gradually reducing it based on a cosine annealing schedule until reaching \(2 \times 10^{-5}\). The training lasts for 150 epochs, with a batch size of 128, on four NVIDIA 4090 GPUs. When constructing a graph for geodesic patch construction, we use kNN with 50 neighbors. 

\para{Datasets.}
We utilize the original training and validation split of the PCPNet dataset~\citep{guerrero2018pcpnet} for our training process. This dataset comprises diverse synthetic point clouds, including both uniform sampling and modified versions. The modified versions exhibit different levels of \emph{noise} and \emph{density} variations. Illustrative examples from this dataset are presented in \emph{Appendix~\ref{appendix:dataset}}. 
After training, we conduct tests on the test sets of the PCPNet dataset. To test the generalization capabilities of our method, we also test on the FamousShape dataset~\citep{li2023shs} and the real-world SceneNN dataset~\citep{li2022hsurf}, without any additional fine-tuning. Further details about these datasets can be found in \emph{Appendix~\ref{appendix:dataset}}. It is worth noting that the FamousShape and SceneNN datasets contain significantly more complex shapes compared to the PCPNet dataset. 

\para{Metrics.} Following~\citep{li2023neural}, we use two metrics to evaluate the model's performance: the \emph{Angle Root Mean Squared Error} (RMSE) and the \emph{Percentage of Good Points} (PGP). The formula for RMSE is defined as follows:
\begin{equation}
\text{RMSE}(\mathbf{n}) = \sqrt{\frac{1}{N} \sum_{i=1}^{N} \arccos^2\left( |\mathbf{n}_i \cdot \Hat{\mathbf{n}}_i|\right)}.
\end{equation}
Here, $\mathbf{n} $ and $ \hat{\mathbf{n}}_is$ are the actual and the predicted angles, respectively. $N$ is the total number of points. The PGP metric incorporates a threshold parameter $ \tau $ and is defined as:
\begin{equation}
 \text{PGP}(\tau) = \frac{1}{N} \sum_{i=1}^{N} \mathbb{I}_{\arccos\left( |\mathbf{n}_i \cdot \Hat{\mathbf{n}}_i|\right) < \tau} \times 100\%, 
\end{equation}
where $ \mathbb{I} $ is an indicator function. PGP assesses the accuracy of normal estimation within a certain range.
% PGP measures the percentage of points where the normal estimation error is within an acceptable range.

% Following~\citep{li2023neural}, we employ the \emph{angle Root Mean Squared Error} (RMSE) to evaluate the model's performance and the \emph{Percentage of Good Points} (PGP) to summarize the normal error distribution. 

\begin{table*}[t]
\centering
%\vskip -0.2cm
\caption{RMSE comparisons on the benchmark PCPNet and SceneNN datasets. \textbf{Bold} fonts indicate the top performer and - indicates that the corresponding source code is unavailable.}
% \vskip -0.3cm
%\scalebox{0.6}{
\renewcommand{\arraystretch}{1.1} 
\begin{adjustbox}{width=\textwidth}
\begin{tabular}{c|c|c|c|c|c|c|c|c||c|c|c}
\hline 
\multirow{3}{*}{Method} & \multirow{3}{*}{Year} & \multicolumn{7}{c||}{PCPNet Dataset} & \multicolumn{3}{c}{SceneNN Dataset} \\ \cline{3-12}
& & \multicolumn{4}{c|}{Noise $\sigma$} & \multicolumn{2}{c|}{Density} & \multirow{2}{*}{Average} & \multirow{2}{*}{Clean} & \multirow{2}{*}{Noise} & \multirow{2}{*}{Average} \\ \cline{3-8}
& & None & $0.12 \%$ & $0.6 \%$ & $1.2 \%$ & Stripes & Gradient & & & & \\
\hline PCA~\citep{hoppe1992surfacepca}& 1992 & 12.29 & 12.87 & 18.38 & 27.52 & 13.66 & 12.81 & 16.25 & 15.93 & 16.32 & 16.12 \\
\hline Jet~\citep{cazals2005estimatingnjet}  & 2005 & 12.35 & 12.84 & 18.33 & 27.68 & 13.39 & 13.13 & 16.29 & 15.17 & 15.59 & 15.38 \\
\hline PCPNet~\citep{guerrero2018pcpnet}  & 2018 & 9.64 & 11.51 & 18.27 & 22.84 & 11.73 & 13.46 & 14.58 & 20.86 & 21.40 & 21.13 \\
\hline Zhou \textit{et al.}~\cite{zhou2020normal}*& 2020 & 8.67 & 10.49 & 17.62 & 24.14 & 10.29 & 10.66 & 13.62 & - & - & - \\
\hline Nesti-Net~\citep{ben2019nesti}& 2019 & 7.06 & 10.24 & 17.77 & 22.31 & 8.64 & 8.95 & 12.49 & 13.01 & 15.19 & 14.10 \\
\hline Lenssen \textit{et al.}~\cite{lenssen2020deep}  & 2020 & 6.72 & 9.95 & 17.18 & 21.96 & 7.73 & 7.51 & 11.84 & 10.24 & 13.00 & 11.62 \\
\hline DeepFit~\citep{ben2020deepfit} & 2020 & 6.51 & 9.21 & 16.73 & 23.12 & 7.92 & 7.31 & 11.80 & 10.33 & 13.07 & 11.70 \\
\hline MTRNet~\citep{cao2021latent}* & 2021 & 6.43 & 9.69 & 17.08 & 22.23 & 8.39 & 6.89 & 11.78 & - & - & - \\
\hline Refine-Net~\citep{zhou2022refine}  & 2022 & 5.92 & 9.04 & 16.52 & 22.19 & 7.70 & 7.20 & 11.43 & 18.09 & 19.73 & 18.91 \\
\hline Zhang \textit{et al.}~\cite{zhang2022geometry}  & 2022 & 5.65 & 9.19 & 16.78 & 22.93 & 6.68 & 6.29 & 11.25 & 9.31 & 13.11 & 11.21 \\
\hline AdaFit~\citep{zhu2021adafit}  & 2021 & 5.19 & 9.05 & 16.45 & 21.94 & 6.01 & 5.90 & 10.76 & 8.39 & 12.85 & 10.62 \\
\hline GraphFit~\citep{li2022graphfit}  & 2022 & 4.45 & 8.74 & 16.05 & 21.64 & 5.22 & 5.48 & 10.26 & 7.99  &12.18  &10.09 \\
\hline NeAF~\citep{li2023neaf} & 2023 & 4.20 & 9.25 & 16.35 & 21.74 & 4.89 & 4.88 & 10.22 & 7.88 & 13.20 & 10.54 \\
\hline HSurf-Net~\citep{li2022hsurf}  & 2022 & 4.17 & 8.78 & 16.25 & 21.61 & 4.98 & 4.86 & 10.11 & 7.55 & 12.23 & 9.89 \\
\hline Du \textit{et al.}~\cite{du2023rethinking} &2023& 4.11 & 8.66 & 16.02 & 21.57 & 4.89 & 4.83 & 10.01 &  7.68 & 11.72 & 9.70 \\
\hline NGL~\citep{li2023neural}& 2023 & 4.06 & 8.70 & 16.12 & 21.65 & 4.80 & 4.56 & 9.98 & 7.74 & 12.26 & 10.00 \\
\hline SHS-Net~\citep{li2023shs}&2023 & 3.95 & 8.55 & 16.13 & 21.53 & 4.91 & 4.67 & 9.96 & 7.93 & 12.40 & 10.17 \\ 

% \hline Wu et al. &2023&3.86& \textbf{8.45} & 16.08 & 21.89 & 4.85 & 4.45&9.93\\
\hline MSECNet~\citep{xiu2023msecnet} &2023& 3.84 & 8.74 & 16.10 & 21.05 & 4.34 & 4.51 & 9.76 & 6.94 & 11.66 & 9.30 \\
\hline Ours & - & $\textbf{3.41}$ & $\textbf{8.54}$ & $\textbf{15.72}$ & $\textbf{20.62}$ &$\textbf{ 3.89}$ & $\textbf{4.05}$ & $\textbf{9.37}$ &$\textbf{6.77}$ &$\textbf{11.34}$ & $\textbf{9.05}$\\
\hline
    % \multicolumn{2}{|c|}{Improvement} & 11.20\%$\uparrow$ & 1.07\%$\downarrow$ & 1.87\%$\uparrow$ & 2.04\%$\uparrow$ & 9.90\%$\uparrow$ & 10.20\%$\uparrow$ & 4.00\%$\uparrow$ \\ 
        \multicolumn{2}{c|}{Improvement} & 11.20\%$\uparrow$ & 0.11\%$\uparrow$ & 1.87\%$\uparrow$ & 2.04\%$\uparrow$ & 10.37\%$\uparrow$ & 10.20\%$\uparrow$ & 4.00\%$\uparrow$ & 2.45\%$\uparrow$ & 2.74\%$\uparrow$ & 2.67\%$\uparrow$ \\

\hline
\end{tabular}
%}
\end{adjustbox}
\label{tab:pcpnet}
% \vskip -0.3cm
\end{table*}
%\subsection{Comparisons.}
\subsection{Comparison to Prior Work}

% \begin{figure*}
%     \centering
%     \includegraphics[width=1\linewidth]{Fig/pcp.pdf}
%         \vskip -0.2cm
%     \caption{Qualitative comparisons of normal estimation on the PCPNet dataset. The values below each model indicate the RMSE deviation.}
%     \label{fig:pcp}
% \end{figure*}

% \begin{figure*}
%     \centering
%     \includegraphics[width=1\linewidth]{Fig/result0.pdf}
%         % \vskip -0.3cm
%     \caption{Qualitative comparisons of normal estimation on the PCPNet and SceneNN dataset. The values below each model indicate the RMSE deviation.}
%     \label{fig:scene}
%     \vskip -0.3cm
% \end{figure*}

\begin{figure*}
    \centering
    \includegraphics[width=1\linewidth]{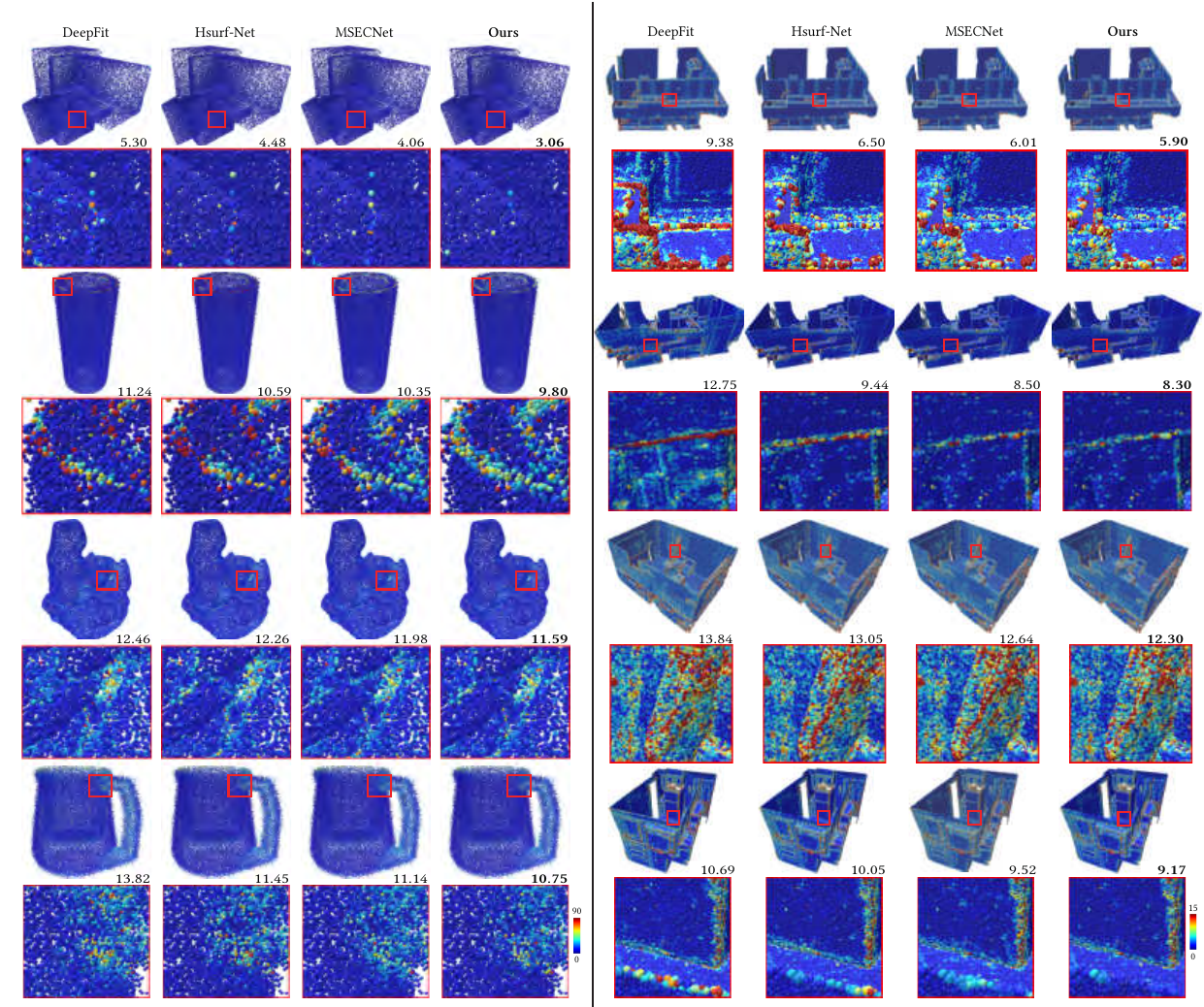}
        % \vskip -0.3cm
    \caption{Qualitative comparisons of normal estimation on the PCPNet and SceneNN dataset. The values below each model indicate the RMSE deviation.}
    \label{fig:scene}
    % \vskip -0.3cm
\end{figure*}

\newcounter{imgcounter}

\newcommand{\imagesort}[1]{\ifcase#1\or 4\or 0\or 1\or 5\or 2\or3\fi} % 
\newcommand{\imagesortn}[1]{\ifcase#1\or 4\or 0\or 1\or 2\or 3\fi} % 
% \begin{figure}

\para{Quantitative and Qualitative Results.} We perform a comprehensive comparative analysis on the PCPNet datasets between our method and various competitors encompassing both traditional and learning-based approaches for normal estimation. As summarized in Table~\ref{tab:pcpnet}, our method surpasses the previously top-performing approach by a significant margin in terms of RMSE.
Under noise-free conditions, we achieve a significant improvement of 11.20\% in RMSE, demonstrating our method's enhanced ability to capture the intrinsic properties of point clouds. For non-uniform sampling, we also achieve over 10\% improvement. Notably, the training data does not include non-uniform sampling data, which shows the generalization ability of our algorithm. Additionally, our method delivers state-of-the-art results in different noise conditions.
On average, we improve RMSE by 4\% on the PCPNet dataset, by 2.67\% on the SceneNN dataset, and 2.44\% on the FamousShape dataset (shown in \emph{Appendix}~\ref{appendix:FamousShape}).
% Concretely, we achieve an impressive 11.20\% improvement under noise-free conditions, indicating its enhanced ability to capture the intrinsic properties of point clouds. This robustness and adaptability are further evident when considering different density conditions. Notably, the error rates under varying density conditions ($3.89$ and $4.05$) are nearly on par with the best previous competitor under noise-free conditions ($3.84$), further highlighting the superiority of our algorithm. Additionally, our method delivers state-of-the-art results across different noise conditions, particularly with a notable 2.04\% improvement in high noise scenarios. In the challenging FamousShape dataset, we also achieve a notable 2.44\% improvement, quantitative comparisons and detailed analysis are reported in \emph{Appendix}. %Table~\ref{tab:famous}. 

Fig.~\ref{fig:pgp} presents the corresponding results using the PGP metric on the PCPNet dataset. The results further validate the significant advantages of our method compared to previous approaches.
We present qualitative comparisons in Fig.~\ref{fig:scene}. From which, it can be seen that our method significantly outperforms previous methods in handling complex details, sharp edges, and boundaries in challenging scenarios.

\para{Timings.} Next, we compare the efficiency of various approaches on the PCPNet dataset. Fig.~\ref{fig:time} highlights the advantages of three different versions of our model, by plotting inference time, parameter count, and accuracy (RMSE) in one figure. \emph{Ours fast} represents a version of $n{_{\mathrm{E}(3)}}=1$ and without GeoPatch. \emph{Ours small} represents a version of $n{_{\mathrm{E}(3)}}=1$, $d_{fused}=256$ and without GeoPatch. \emph{Ours} stands as the default version, utilizing GeoPatch.
It is noteworthy that all three versions excel in terms of accuracy compared to previous methods. The complete version exhibits the most significant improvement in accuracy. Moreover, the fast version displays considerable enhancements in both accuracy and speed. Meanwhile, the small version, despite having fewer parameters than previous methods, still stands out for its remarkable speed and precision.

% \para{Real-World Dataset.} Table~\ref{tab:pcpnet} also reports a notable 2.67\% improvement achieved by our method on the real-world SceneNN dataset. This result not only demonstrates the practical effectiveness of our approach but also highlights its significant enhancements in handling unseen data distributions during training, under both clean and noisy conditions. It underscores the robust generalization capabilities of our method. Overall, our approach outperforms existing methods in addressing complex point cloud challenges in both synthetic and real-world scenarios. 

\begin{table*}[t]
   \centering
   \caption{Ablation studies on the PCPNet dataset. Our default setting achieves the best average result.} %\textbf{Bold} fonts indicate the best setting.}
   % \vskip -0.3cm
   %\scalebox{0.65}{
\begin{adjustbox}{width=\textwidth}
\begin{tabular}{|c|c|c|c|c|c|c|c|c|c|}
\hline 
\multicolumn{2}{|c|}{\multirow{2}{*}{Ablation}} & \multicolumn{4}{c|}{Noise $\sigma$} & \multicolumn{2}{c|}{Density} & \multirow{2}{*}{Average} & \multirow{2}{*}{$\Delta$} \\
\cline{3-8}
\multicolumn{2}{|c|}{} & None & $0.12 \%$ & $0.6 \%$ & $1.2 \%$ & Stripes & Gradient & & \\
% \hline 
% 其余行...
\hline Default & \begin{tabular}{c}
$N=1400$, $d_{f u s e d}=1024$, $n{_{\mathrm{E}(3)}}=8$, w/ $\mathcal{L}_{\text{Gaussian}}$, \\
w/ GeoPatch, w/ HalfPatch, w/ GaussianPatch,\\w/  Random Frame
\end{tabular} &3.41 &8.54 & 15.72& 20.62 & 3.89&4.05&9.37& - \\
\hline
\multirow{2}{*}{ (a) } 

& w/  Random Frame, $d_{f u s e d}=128$ & 3.75&8.56& 15.85& 20.74& 4.42&4.50& 9.64 &-0.27\\
\cline{2-10}
& w/o  Random Frame, $d_{f u s e d}=128$   & 4.49 & 8.71 & 16.04 &  20.91   &  5.28   &  5.25   &  10.11&-0.74     \\
\cline{2-10}
\hline
\multirow{3}{*}{ (b) } & $n{_{\mathrm{E}(3)}}=4$   &  3.41   & 8.56 & 15.71 &  20.64  &  3.93   &  4.09   &  9.39& -0.02   \\
\cline{2-10}
&  $n{_{\mathrm{E}(3)}}=2$  & 3.48& 8.56&15.75& 20.64&3.94& 4.11& 9.41 &-0.04\\
\cline{2-10}
& $n{_{\mathrm{E}(3)}}=1$   & 3.51&8.61&15.78&20.68& 3.97&4.16& 9.45&-0.08 \\
% \cline{2-10}
% 3.51 8.61 15.78 20.68 3.97 4.16 9.45
%  3.48  8.56  15.75    20.64  3.94  4.11 9.41
% 3.41 8.56 15.71 20.64 3.93 4.09 9.39
% &  $n{_{\mathrm{E}(3)}}=1$  & 3.49&8.59&15.84& 20.68& 4.06&4.15& 9.48&-0.11 \\
\hline
\multirow{2}{*}{ (c) } 
& w/  $\mathcal{L}^{Val}$    & 3.48 & {8.50} & 15.78 &  20.64   &  4.02   &  4.20   &  9.44&-0.07     \\
\cline{2-10}
& w/  $\mathcal{L}^{Half}$   &3.52 & {8.50} & {15.68} &  20.69   &  4.33   &  4.15   &  9.48 &-0.11\\
% \cline{2-10}

% [3.993204, 8.560199, 15.849955, 20.681103, 4.7082777, 4.6945963] | Mean: 9.747889518737793
% & w/o GeoPatch, $n{_{E(3)}}=2$   & 3.45&8.53&15.81&20.63& 4.05&4.17& 9.44&-0.07 \\
% \cline{2-10}
% & w/o GeoPatch, $n{_{E(3)}}=1$  & 3.49&8.59&15.84& 20.68& 4.06&4.15& 9.48&-0.11 \\
\hline
% \multirow{2}{*}{ (b) } 
% & w/  $\mathcal{L}^{vallina}$    & 3.48 & 8.50 & 15.78 &  20.64   &  4.02   &  4.20   &  9.44&-0.07     \\
% \cline{2-10}
% & w/  $\mathcal{L}^{Half}$   &3.52 & 8.50 & 15.68 &  20.69   &  4.33   &  4.15   &  9.48 &-0.11\\
% \cline{2-10}
% & w/o GeoPatch, $n{_{E(3)}}=2$   & 3.45&8.53&15.81&20.63& 4.05&4.17& 9.44&-0.07 \\
% \cline{2-10}
% & w/o GeoPatch, $n{_{E(3)}}=1$  & 3.49&8.59&15.84& 20.68& 4.06&4.15& 9.48&-0.11 \\
% \hline
 % [3.425510945405495, 8.549816223598391, 15.727541087036977, 20.63333662609752, 3.909316734256253, 4.071198119283575] | Mean: 9.386119955946368
  % [3.4998731770994986, 8.661598578854242, 15.900277855649401, 20.8149188717001, 4.064099014540125, 4.1557234797696765] | Mean: 9.516081829602173
\multirow{3}{*}{ (d) } 
& w/o GeoPatch &  3.45   & 8.52 & 15.74 &  {20.61}   &  3.96   &  4.12   &  9.40& -0.03   \\
\cline{2-10}
& w/o HalfPatch &  3.50   & 8.66 & 15.90 &  {20.81}   &  4.06   &  4.16   &  9.52  &-0.15\\
\cline{2-10}
& w/o GaussianPatch&  3.43   & 8.55 & 15.73 &  20.63   &  3.91  &  4.07   &  9.39& -0.02\\
\hline
\multirow{2}{*}{ (e) } 
&$N=700$  & {3.37}& 8.53& 15.74&20.96&{3.87}& {3.98}& 9.41&-0.04     \\
\cline{2-10}
& $N=2100$  & 3.47&8.56&15.79&20.63&3.98&4.09& 9.42 &-0.05\\
% \cline{2-10}
% & w/o GeoPatch, $n{_{E(3)}}=2$   & 3.45&8.53&15.81&20.63& 4.05&4.17& 9.44&-0.07 \\
% 3.51 8.61 15.78 20.68 3.97 4.16 9.45
%  3.48  8.56  15.75    20.64  3.94  4.11 9.41
% 3.41 8.56 15.71 20.64 3.93 4.09 9.39
% \cline{2-10}
% & w/o GeoPatch, $n{_{E(3)}}=1$  & 3.49&8.59&15.84& 20.68& 4.06&4.15& 9.48&-0.11 \\
\hline
\multirow{2}{*}{ (f) } 
& $d_{f u s e d}=512$  &  3.50&8.49&15.70&20.70& 4.10& 4.21&9.45 &-0.08     \\
\cline{2-10}
& $d_{f u s e d}=256$ & 3.66&8.57& 15.71& 20.70& 4.30&4.44& 9.56 &-0.19\\

 % [3.7461229884228837, 8.559256070244396, 15.850010930303016, 20.73661645842713, 4.422496973847584, 4.501808715741342] | Mean: 9.636052022831059

%  /data/yssong/MSECNETor/scripts/results_ssss/2_4/last/pcpnet 
%  All RMS not oriented (shape average): [3.4122499270365965, 8.558225647453135, 15.712149780316993, 20.640221983762526, 3.92803344324901, 4.08401738546758] | Mean: 9.389149694547639

%  /data/yssong/MSECNETor/scripts/results_ssss/2_1/last/pcpnet 
%  All RMS not oriented (shape average): [3.5172733549515924, 8.607600408927581, 15.775014963791905, 20.67714028664563, 3.968486417754325, 4.158375771069392] | Mean: 9.450648533856738

% % 3.51 8.61 15.78 20.68 3.97 4.16 9.45
% %  3.48  8.56  15.75    20.64  3.94  4.11 9.41
% 3.41 8.56 15.71 20.64 3.93 4.09 9.39

% : [4.4880677099782, 8.708820438292724, 16.03667932139791, 20.91447123337042, 5.278559869551497, 5.2541192492562985] | Mean: 10.113452970307842

%  [3.7461229884228837, 8.559256070244396, 15.850010930303016, 20.73661645842713, 4.422496973847584, 4.501808715741342] | Mean: 9.636052022831059

% \cline{2-10}
% & w/o GeoPatch, $n{_{E(3)}}=2$   & 3.45&8.53&15.81&20.63& 4.05&4.17& 9.44&-0.07 \\
% \cline{2-10}
% & w/o GeoPatch, $n{_{E(3)}}=1$  & 3.49&8.59&15.84& 20.68& 4.06&4.15& 9.48&-0.11 \\
\hline

    % \hline Ours  dimfnn512  &  3.50&8.49&15.70&20.70& 4.10& 4.21&9.45  \\
    %  Ours  dimfnn256  & 3.66&8.57& 15.71& 20.70& 4.30&4.44& 9.56\\
    % \hline OurstestPATCHSIZE700  & 3.37& 8.53& 15.74&20.96&3.87& 3.98& 9.41\\
    % OurstestPATCHSIZE2100 &  3.47&8.56&15.79&20.63&3.98&4.09& 9.42\\

\end{tabular}
\end{adjustbox}
%}
\label{tab:ablatioin}
% \vskip -0.5cm
\end{table*}

\begin{figure*}[t]
    \centering
    \includegraphics[width=1\linewidth]{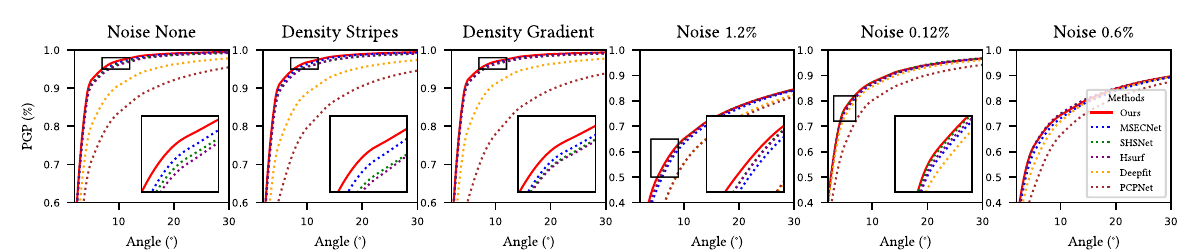}
    % \vskip -0.4cm
    \caption{Statistics of the PGP results on the PCPNet dataset. }
    \label{fig:pgp}
    % \vskip -0.3cm
\end{figure*}
% \begin{figure}[h]
%     \centering
%     % \vskip -0.3cm
%     \includegraphics[width=0.5\linewidth]{Fig/time_time1.pdf}
%     \vskip -0.3cm
%     \caption{Comparative analysis of runtime and angle RMSE error across different methods with circle size indicating the model parameters.}
%     \label{fig:time}
%     \vskip -0.3cm
% \end{figure}
\begin{figure}[!bp]
    \centerline{
    \begin{minipage}{0.4\textwidth}
        \centering
        \includegraphics[height=1.8in]{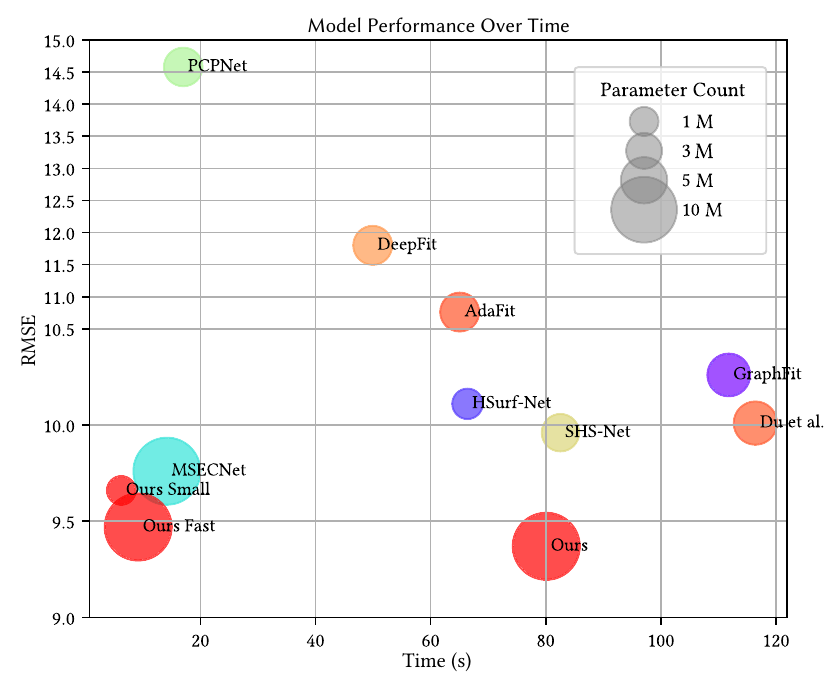}
        % \vskip -0.3cm
        \caption{Comparative analysis of runtime and angle RMSE error across different methods with circle size indicating the model parameters.}
        \label{fig:time}
    \end{minipage}
    
    \hfill
    
    \begin{minipage}{0.54\textwidth}
    \centering
    \vskip -0.35cm
    \includegraphics[height=1.8in]{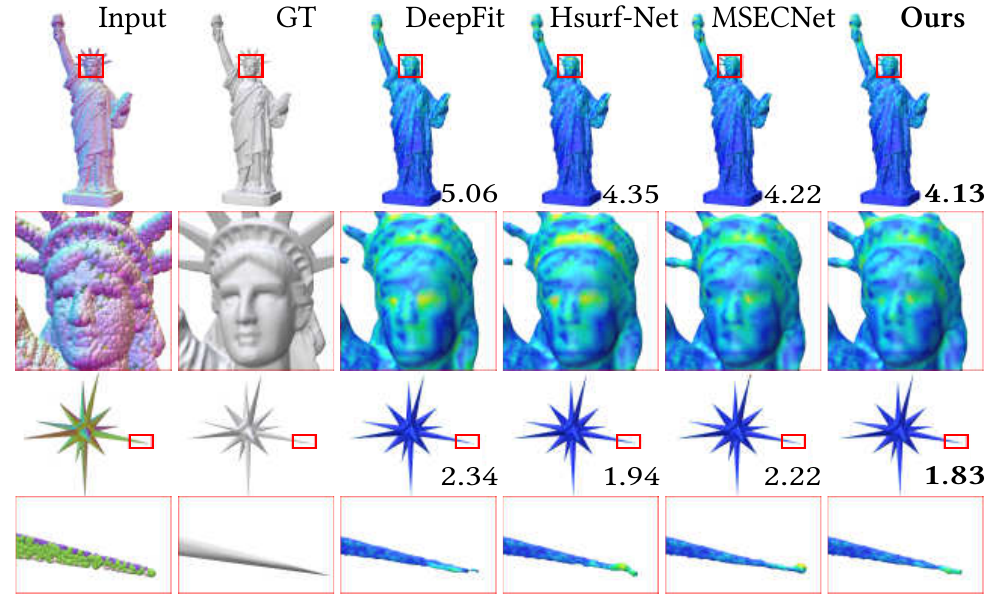}
     % \vskip -0.3cm
    \caption{Comparisons of surface reconstruction. The values are the Hausdorff distance to the ground truth surface.}%{(scaled as the definition in Fig.~\ref{fig:teaser})}.}
%The meanings of the numbers in the figure are the same as in the second line of Fig.~\ref{fig:teaser}}
    \label{fig:recon}
    % \vskip -0.5cm
    \end{minipage}
    }
\end{figure}

\subsection{Ablation Studies}
In this part, we perform ablation studies to validate various components of our model and assess their impact on performance.

% \begin{figure}[!htbp]
%     \centering
%     % \vskip 0.3cm
%     \includegraphics[width=1\linewidth]{Fig/mesh1.pdf}
%      \vskip -0.3cm
%     \caption{Comparisons of surface reconstruction. The values are the Hausdorff distance to the ground truth surface {(scaled as the definition in Fig.~\ref{fig:teaser})}.}
% %The meanings of the numbers in the figure are the same as in the second line of Fig.~\ref{fig:teaser}}
%     \label{fig:recon}
%     \vskip -0.5cm
% \end{figure}
\para{(a) Random Frame Strategy.} 
To demonstrate the efficacy of our algorithm's random strategy from Sec.~\ref{subsec:E3Equivarance}, we utilize a configuration with setting feature fusion dimension ($d_{f u s e d}$) 128, and compare it against Puny \textit{et al.}~\cite{Puny_2021frameaverage}, which does not employ the random strategy. With the same parameter settings, replacing the random strategy with Puny \textit{et al.}~\cite{Puny_2021frameaverage}  requires eight times more training resources and results in 3.8\% worse RMSE. With the same training resources, compared to our complete model, Puny \textit{et al.}~\cite{Puny_2021frameaverage} shows an accuracy decrease of 7.9\%. These results clearly illustrate that our Random Strategy significantly surpasses Puny \textit{et al.}~\cite{Puny_2021frameaverage} in terms of both accuracy and efficiency.

\para{(b) Frame Element Selection.} 
As reported in Table~\ref{tab:ablatioin}, experiments reveal that using a larger number of randomly selected elements from the frame set during inference %(without using GeoPatch)
leads to improved performance. This finding implies the significance of E(3) in enhancing the efficacy of our model.

\para{(c) Loss Function Variation.} 
To verify the superiority of our loss function, we introduce a new function, $\mathcal{L}^{Half}$, which focuses the optimization on the nearest half of the points in a patch. $\mathcal{L}^{Val}$ means equal optimization across all points. While $\mathcal{L}^{Half}$ shows advancements in scenarios with medium noise, it is less effective compared to $\mathcal{L}^{Val}$, particularly with a 0.31 decrease in Stripes. This limitation is attributed to $\mathcal{L}^{Half}$'s limited utilization of information from unevenly sampled patches. This contrast emphasizes the strength of our Gaussian weighting approach, which optimizes neighborhood information more effectively.

\para{(d) Configuration Removal.} We remove key configurations including GeoPatch (geodesic patching), Half Sampling (discarded boundary points), and GaussianPatch (Gaussian weighted sampling) to evaluate their impacts. The results presented in Table~\ref{tab:ablatioin} confirm the effectiveness of each component in our inference strategy.

\para{(e) Patch Size.} We also evaluate various patch sizes. Results reveal that the optimal patch size for efficient and effective noise and density handling is 1400.

\para{(f) Feature Fusion Dimension Adjustment.} By altering $d_{f u s e d}$, we observe a clear relationship between the number of parameters and model performance. This experiment highlights the impact of the feature fusion dimension  on the overall performance of the model.

\subsection{Applications}
We demonstrate the effectiveness of E$^3$-Net via Poisson surface reconstruction~\citep{kazhdan2013screened} (Fig.~\ref{fig:recon}). As observed in Fig.~\ref{fig:recon}, our approach excels in challenging regions, including model tips and intersections, showcasing its ability to capture intricate surface structures. We also demonstrate the versatility of our proposed framework by applying random frame and frame averaging techniques to \emph{point cloud denoising} (please refer to the \emph{Appendix}~\ref{appendix:denoise} for more details).

\section{Conclusions, Limitations, and Future Work} We presented a novel method that guarantees E(3)-equivariance for normal estimation in point clouds. Our approach effectively balances high-quality results and efficiency by combining random frame training with average frame inference. Moreover, we leverage a Gaussian weighted loss, geodesic patches, and receptive-aware weighting to enhance local feature aggregation, resulting in impressive performance. Our method outperforms competitors by a remarkable margin on three datasets and demonstrates the capability to address practical applications.

While \emph{Ours small} version stands out as the fastest, smallest, and most accurate among all previous algorithms, it still exhibits some limitations compared to our complete version. %Besides, our method does not consider the intrinsic noise present in point clouds. This suggests potential avenues for improvements such as incorporating denoising processes to further enhance algorithmic accuracy. 
Although our method achieves state-of-the-art results on real-world data, it only relies on synthetic data for training. In the future, constructing larger real-world datasets with various scanning noises and developing better simulators for real-world scanners could be beneficial. By training on better synthetic data and fine-tuning on real data, the performance could be further enhanced.
Additionally, the proposed method is also adaptable for various patch-based point cloud analysis tasks, such as point cloud resampling, point cloud completion, and curvature estimation. %Additionally, the method we proposed can be applied to any patch-based point cloud tasks such as point cloud resampling and point cloud completion. %Future work could explore the integration of our method with advanced point cloud processing techniques.

\para{Broader Impact.} 
Our algorithm addresses a fundamental problem in geometry processing, making it applicable to various domains such as point cloud denoising and surface reconstruction. As a low-level building block, our work has no direct negative societal impacts.

% Our algorithm focuses on a fundamental problem in geometric processing and thus has various applications ranging from point cloud denosing to surface reconstruction, promoting impressive progress in this area. As a low-level building block, our work has no direct negative outcome, other than what could arise from the aforementioned applications.

% \newpage
% \bibliographystyle{ACM-Reference-Format}
{\small
	\bibliographystyle{unsrt}
	\bibliography{neurips_2024.bib}
}

\newpage
\appendix
\section*{\centering{Appendix}}
\section{Theoretical Proof}
\label{appendix:proof}
In the following section, we present a theoretical proof to substantiate Frame Avereage in $\mathbb{R}^{3}$, which is more concise compared to Puny \textit{et al.}~\cite{Puny_2021frameaverage}. 

\newtheorem{theorem}{Theorem}

\begin{theorem}
Let $\mathbf{x}\in\mathbb{R}^{m\times 3}$ be an input point patch and $f\in \mathrm{E}(3)$ be an operation. Then the function $\Phi$ defined in Eq.~\ref{eq:frame_ave} satisfies  
\begin{equation}
\Phi(f(\mathbf{x})) = f \Phi(\mathbf{x})   
\end{equation}

\end{theorem}

\begin{proof}
Assume $\mathscr{F}(\mathbf{x})$ is defined as follows: 
\begin{align}
 \mathscr{F}(\mathbf{x})&=   \left\{ \left( [\pm \boldsymbol{v}_1, \pm\boldsymbol{v}_2, \pm \boldsymbol{v}_3], \boldsymbol{t} \right)  \right\}  \\
&=\{(\mathbf{U} \mathbf{T}_i, t) \mid \mathbf{x}=\mathbf{U} \Lambda \mathbf{V}^{\top}, \mathbf{t}=\frac{1}{m}\mathbf{x}^{\top}\mathbf{1}_{m\times 1}, \mathbf{T}_i \in  \mathrm{D}\}. 
\end{align}where $\mathbf{U \Lambda V^{\top}}$ is the singular value decomposition of the matrix $\mathbf{x}$, $\mathbf{D}=\{[\pm \boldsymbol{v}_1, \pm\boldsymbol{v}_2, \pm \boldsymbol{v}_3]\}$and $\mathbf{t}$ is the mean of $\mathbf{x}$ along its second dimension.

Next, assume a spatial transformation $f = (\mathbf{U}_1, \mathbf{t}_1) \in \mathrm{E}(3)$. When applying $f$ to $\mathbf{x}$, $\mathscr{F}(\mathbf{x})$ is modified to $\mathscr{F}(f(\mathbf{x}))$, resulting in: 
\[ \mathscr{F}(f(\mathbf{x})) = \left\{ ( \mathbf{U}_1  \mathbf{UT}_i,  \mathbf{t}_1 + \mathbf{t}) \right\} \]

For the calculation of $\Phi(f(\mathbf{x}))$, the function $\Phi$ applied to $f(\mathbf{x})$ is defined as: 
\[ \Phi(f(\mathbf{x})) = \frac{1}{|\mathcal{F}(f(\mathbf{x}))|} \sum_{g_i \in \mathscr{F}(f(\mathbf{x}))} g_i \phi(g_i^{-1} f(\mathbf{x})) \]

Substituting the above definitions into $\Phi(f(\mathbf{x}))$ we get: 
\begin{align}
\Phi(f(\mathbf{x})) &= \frac{1}{8} \sum_{\mathbf{T}_i \in \mathrm{ D}} \left( \mathbf{U}_1 \mathbf{U}\mathbf{T}_i,  \mathbf{t}_1 + \mathbf{t}\right) \phi\left(\left(\mathbf{T}_i \mathbf{U}^{\top}  \mathbf{U}_1^{\top}, - \mathbf{t_1} - \mathbf{t}\right) \left( \mathbf{U}_1,  \mathbf{t}_1\right)(\mathbf{x})\right) \\
&= \frac{1}{8} \sum_{\mathbf{T}_i \in  \mathrm{D}} \left( \mathbf{U}_1,  \mathbf{t}_1\right)\left(\mathbf{U T}_i, \mathbf{t}\right) \phi\left(\left(\mathbf{T}_i \mathbf{U}^{\top}, -\mathbf{t}\right)(\mathbf{x})\right) \\
&= \frac{1}{8} \sum_{\mathbf{T}_i \in  \mathrm{D}} f g_i \phi(g_i^{-1}(\mathbf{x})) \\
&= f \frac{1}{|\mathcal{F}(\mathbf{x})|} \sum_{g_i \in \mathscr{F}(\mathbf{x})} g_i \phi(g_i^{-1}(\mathbf{x})) \\
&= f \Phi(\mathbf{x})
\end{align} 
This concludes the proof.
\end{proof}

\section{Details of the Used Datasets}
\label{appendix:dataset}
We give a detailed description of the used datasets in the paper. 
\begin{figure}[H]
    \centering
    \includegraphics[width=1\linewidth]{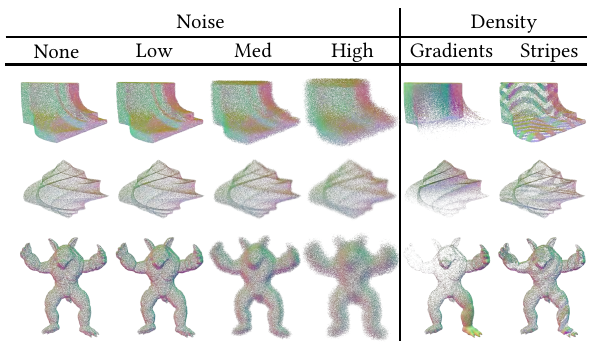}
    \caption{Examples of point clouds from the PCPNet Dataset.}
    \label{fig:enter-label4}
\end{figure}
\para{PCPNet.} The PCPNet dataset~\cite{guerrero2018pcpnet} is divided into three sets: training, validation, and test sets. The training set comprises eight models, each available in noise-free, high, medium, and low noise versions, resulting in a total of 32 point clouds with consistent density. The validation set consists of three models, each available in noise-free, high, medium, and low noise versions, as well as stripes and gradient variable density versions, resulting in a total of 18 point clouds. The test set applies these transformations to 18 models, resulting in a total of 108 point clouds for evaluation. Fig. \ref{fig:enter-label4} presents several illustrative examples from PCPNet dataset. 

% Our method involves exclusive training on the training set of the PCPNet dataset. Subsequent tests are conducted directly on the test sets of the PCPNet dataset, as well as the FamousShape~\cite{li2023shs} dataset, and the real-world SceneNN~\cite{li2022hsurf} dataset. This testing protocol is executed without any additional fine-tuning of the model.

\begin{table*}[t]
  \centering
\caption{Quantitative comparisons of normal estimation on the FamousShape dataset. \textbf{Bold} fonts indicate the top performer.}
\begin{adjustbox}{width=0.9\textwidth}
\begin{tabular}{c|c|c|c|c|c|c|c|c}
    \hline 
    \multirow{3}{*}{Method} & \multirow{3}{*}{Year} & \multicolumn{7}{c}{FamousShape Dataset} \\ \cline{3-9}
    & & \multicolumn{4}{c|}{Noise $\sigma$} & \multicolumn{2}{c|}{Density} & \multirow{2}{*}{Average} \\ \cline{3-8}
    & & None & $0.12 \%$ & $0.6 \%$ & $1.2 \%$ & Stripes & Gradient & \\
\hline Jet \cite{cazals2005estimatingnjet}  &  2005 & 20.11 & 20.57 & 31.34 & 45.19 & 18.82 & 18.69 & 25.79 \\
\hline PCA \cite{hoppe1992surfacepca}  & 1992 & 19.90 & 20.60 & 31.33 & 45.00 & 19.84 & 18.54 & 25.87 \\
\hline PCPNet \cite{guerrero2018pcpnet}  & 2018 & 24.71 & 28.00 & 40.26 & 49.78 & 25.98 & 26.12&32.47 \\
% \hline Zhou et al.  &[Zhou et al. 2020b] & - & - & - & - & - & - & - \\
\hline Nesti-Net \cite{ben2019nesti}  & 2019 & 11.60 & 16.80 & 31.61 & 39.22 & 12.33 & 11.77 & 20.55 \\
\hline Lenssen et al. \cite{lenssen2020deep} & 2020 & 11.62 & 16.97 & 30.62 & 39.43 & 11.21 & 10.76 & 20.10 \\
\hline DeepFit \cite{ben2020deepfit} & 2020 & 11.21 & 16.39 & 29.84 & 39.95 & 11.84 & 10.54 & 19.96 \\
% \hline MTRNet  &[Cao et al. 2021] & - & - & - & - & - & - & - \\
% \hline Refine-Net  &[Zhou et al. 2022] & - & - & - & - & - & - & - \\
\hline Zhang et al. \cite{zhang2022geometry}  &2022 & 9.83 & 16.13 & 29.81 & 39.81 & 9.72 & 9.19 & 19.08 \\
% \hline Zhou et al.  &[Zhou et al. 2023] & - & - & - & - & - & - & - \\
\hline AdaFit \cite{zhu2021adafit} & 2021 & 9.09 & 15.78 & 29.78 & 38.74 & 8.52 & 8.57 & 18.41 \\
\hline GraphFit \cite{li2022graphfit}  & 2022 & 8.91 & 15.73 & 29.37 & 38.67 & 9.10 & 8.62 & 18.40 \\
\hline NeAF \cite{li2023neaf} & 2023 & 7.67 & 15.67 & 29.75 & 38.76 & 7.22 & 7.47 & 17.76 \\
\hline HSurf-Net \cite{li2022hsurf} &2022 & 7.59 & 15.64 & 29.43 & 38.54 & 7.63 & 7.40 & 17.70 \\
\hline NGLO \cite{li2023neural}& 2023& 7.25 & 15.60 & 29.35 & 38.74 & 7.60 & 7.20 & 17.62 \\
\hline SHS-Net \cite{li2023shs} &2023 & 7.41 & 15.34 & 29.33 & 38.56 & 7.74 & 7.28 & 17.61 \\
\hline MSECNet \cite{xiu2023msecnet} &2023& 6.86 & 15.54 & 29.24 & 38.16 & 6.64 & 6.71 & 17.19 \\
\hline Ours &-& $\textbf{6.45}$ & $\textbf{15.03}$  & $\textbf{28.76}$  & $\textbf{37.75}$  &  
$\textbf{6.31}$  &$\textbf{ 6.30}$  &$\textbf{ 16.77}$  \\
\hline \multicolumn{2}{c|}{Improvement} & 5.98\%$\uparrow$ & 2.02\%$\uparrow$ & 1.64\%$\uparrow$ & 1.07\%$\uparrow$ & 4.97\%$\uparrow$ & 6.11\%$\uparrow$& 2.44\%$\uparrow$ \\ 
\hline

% [6.305293812471038, 6.311851185374287, 6.45449044021852, 37.7543201071662, 15.033605741205937, 28.76580017590352] | Mean: 16.770893577056583
% [6.713078141105221, 6.640532981289575, 6.856936353372905, 38.163801134761755, 15.540559293802145, 29.2438612542922] | Mean: 17.193128193103966
  \end{tabular}
\end{adjustbox}
  \label{tab:famous}

\end{table*}
\para{FamousShape.}
In comparison to the PCPNet dataset, the FamousShape dataset~\cite{li2023shs} incorporates more intricate structures obtained from various public datasets, such as the Famous dataset~\cite{erler2020points2surf} and the Stanford 3D Scanning Repository~\cite{curless1996volumetric}. The FamousShape dataset undergoes the same preprocessing steps as the PCPNet dataset to introduce data contamination.

\para{SceneNN.}
SceneNN is a real-world dataset that comprises both clean and noisy scenes sourced from SceneNN~\cite{hua2016scenenn}, prepared by \cite{li2022hsurf}.

\section{Assessments on the FamousShape Dataset}
\label{appendix:FamousShape}
\begin{figure*}
    \centering
    \includegraphics[width=1\linewidth]{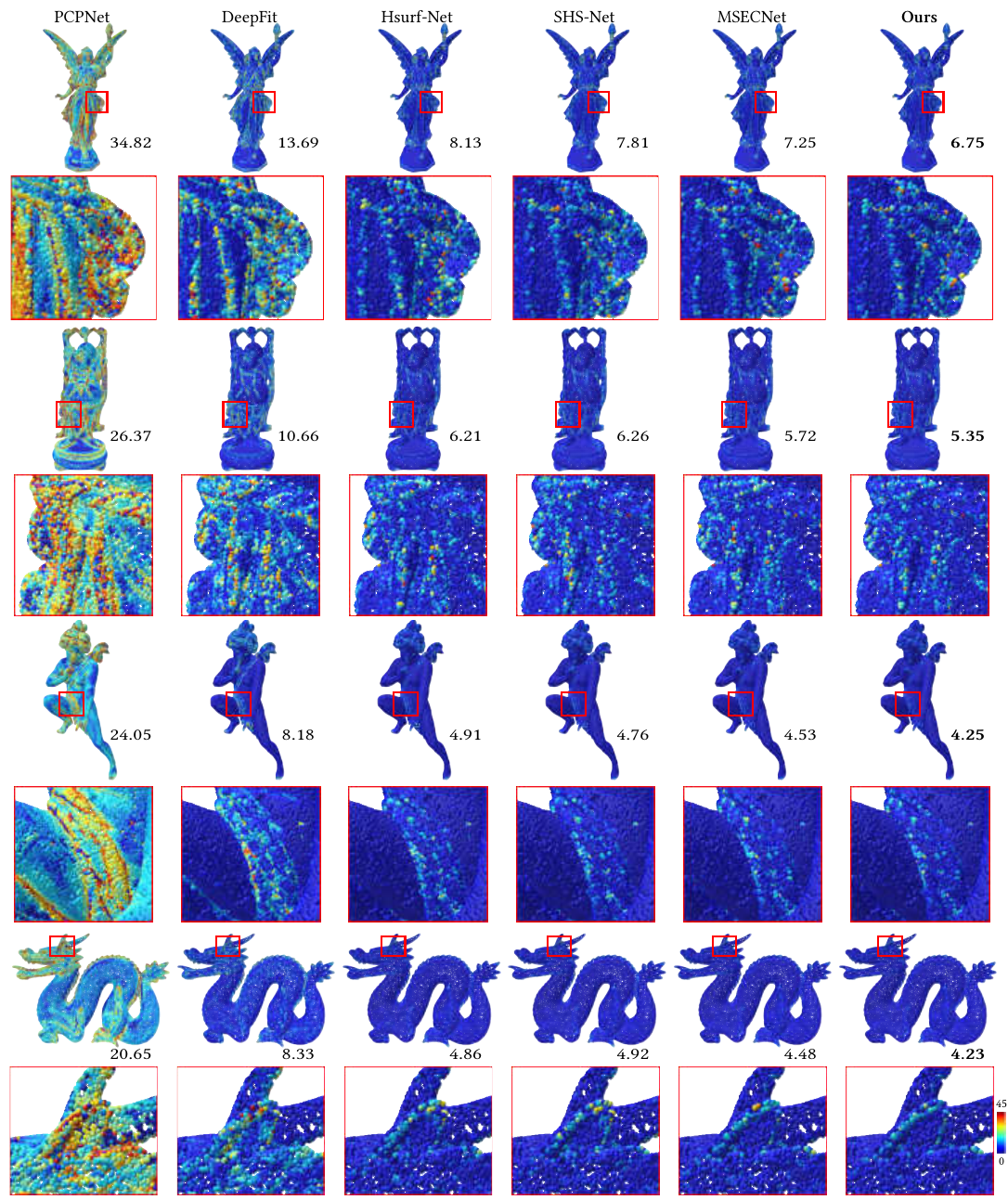}
        %\vskip -0.3cm
    \caption{Qualitative comparisons of normal estimation on the FamousShape dataset. The values below each model indicate the RMSE deviation.}
    \label{fig:famous1}
\end{figure*}

\begin{figure*}
    \centering
    \includegraphics[width=1\linewidth]{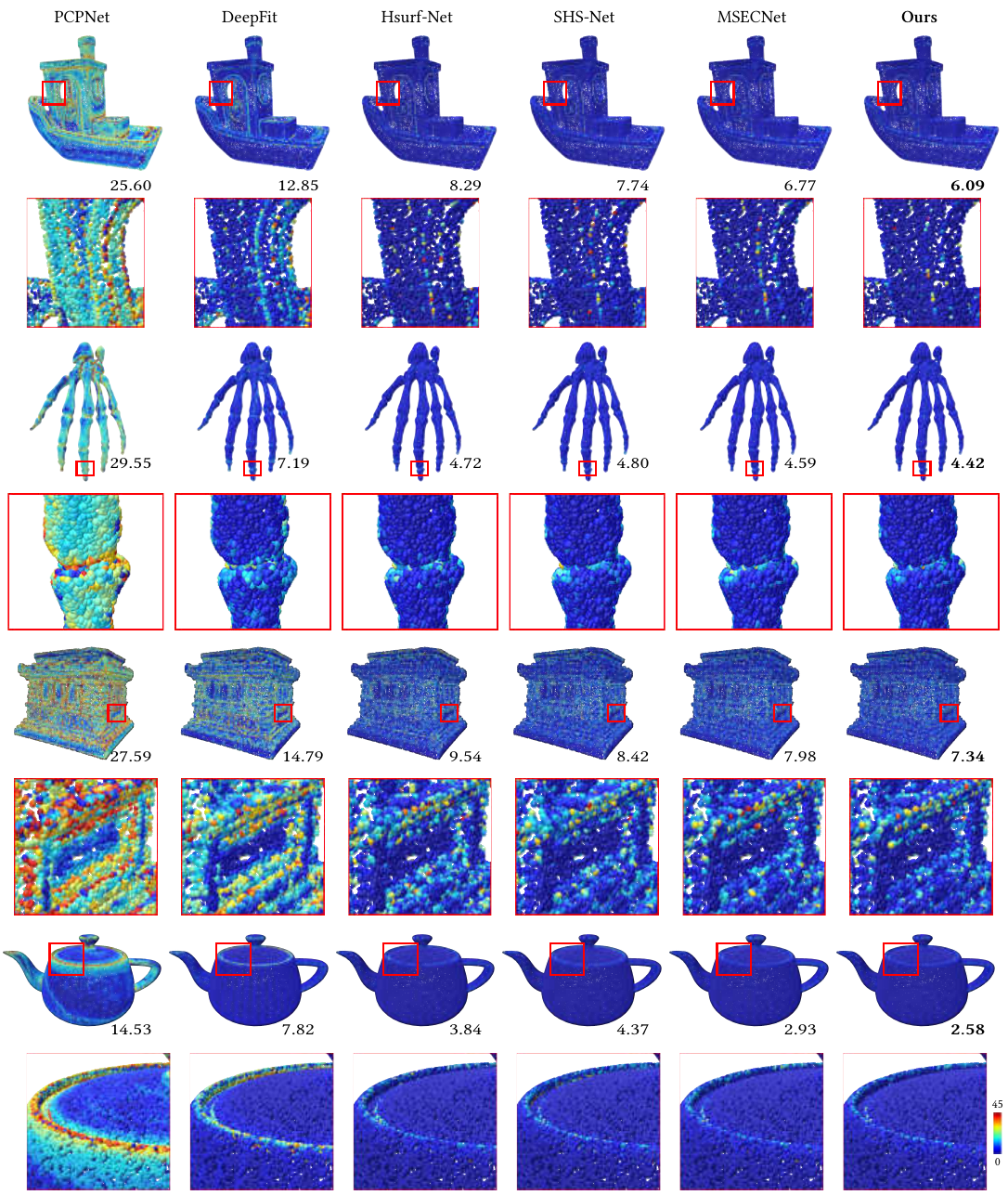}
    %\vskip -0.3cm
    \caption{Qualitative comparisons of normal estimation on the FamousShape dataset. The values below each model indicate the RMSE deviation.}
    \label{fig:famous2}
\end{figure*}

\begin{figure*}
    \centering
    \includegraphics[width=1\linewidth]{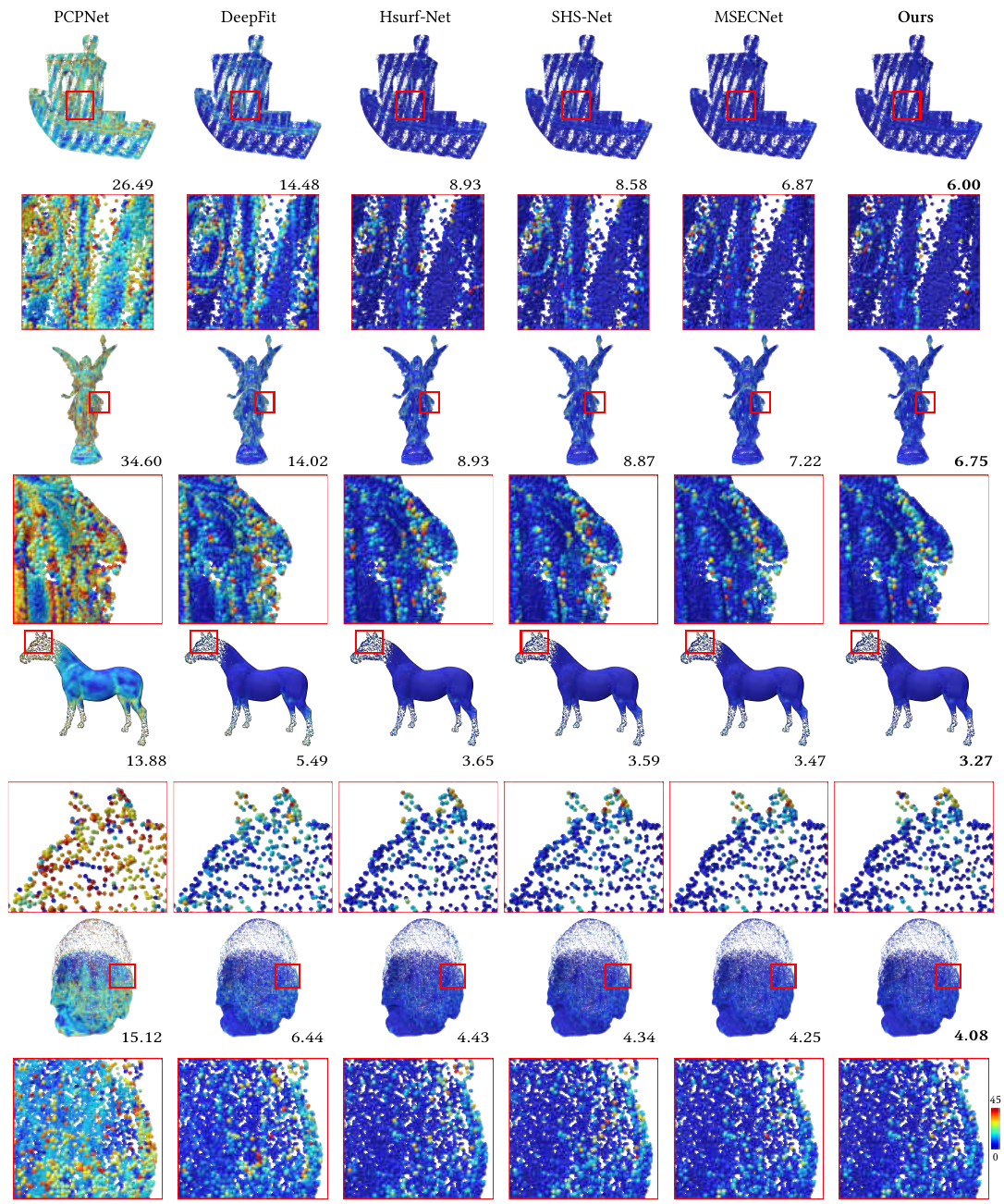}
    %\vskip -0.3cm
    \caption{Qualitative comparisons of normal estimation of non-uniform density point clouds in the FamousShape dataset. The values below each model indicate the RMSE deviation.}
    \label{fig:famous3}
\end{figure*}

\begin{figure*}
    \centering
    \includegraphics[width=1\linewidth]{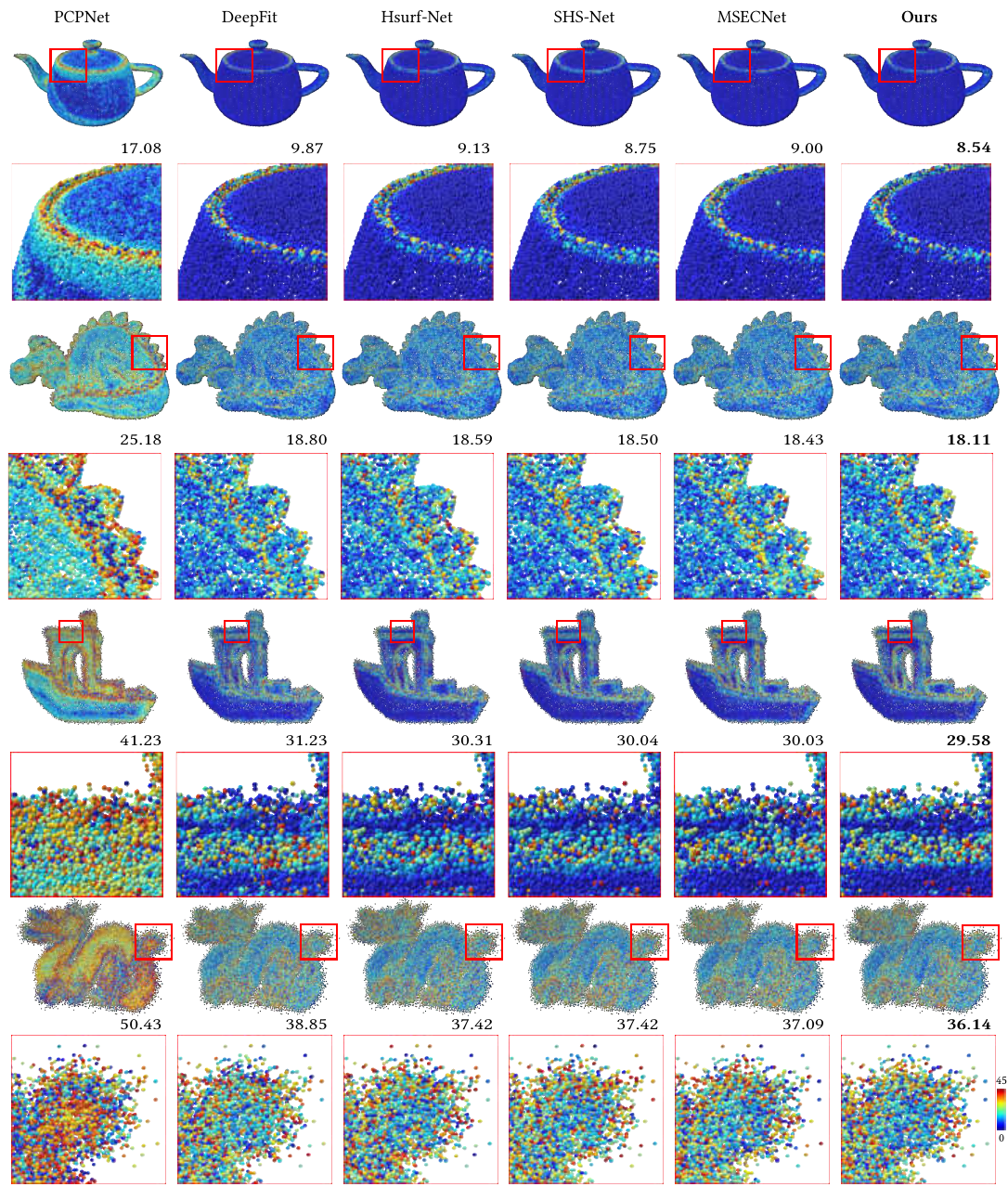}
    %\vskip -0.3cm
    \caption{Qualitative comparisons of normal estimation of noisy point clouds in the FamousShape dataset. The values below each model indicate the RMSE deviation.}
    \label{fig:famous4}
\end{figure*}
In our paper, we have conducted a thorough evaluation of the proposed algorithm under various conditions. Additionally, we have performed assessments on the more challenging FamousShape dataset. When evaluated under noise-free and uniform sampling settings, as depicted in Figs.~\ref{fig:famous1}-\ref{fig:famous2}, the algorithm demonstrates significant improvements in processing complex human and animal models. Notably, enhancements can be observed in intricate aspects such as clothing folds in the Happy, Angel, and Lucy models, as well as the fine edges of the dragon's ears, highlighting the algorithm's ability to handle detailed structures effectively.

In gradient sampling scenarios, as depicted in Fig.~\ref{fig:famous3}, the proposed algorithm demonstrates exceptional capability, particularly in regions with noticeable variations in sampling density, as exemplified by the Serapis model. It significantly enhances finer details, such as the ears of the Horse model and the sparsely covered head. These results highlight the algorithm's precision and effectiveness in handling challenging gradient sampling scenarios. Additionally, in the case of stripe sampling, as illustrated in Fig.~\ref{fig:famous3}, the algorithm successfully handles complex areas such as Benchy's wall surfaces and the sparsely covered skirt, demonstrating consistent performance and accuracy even in challenging conditions.

 Lastly, in challenging noisy environments, as depicted in Fig.~\ref{fig:famous4}, the algorithm effectively handles intricate areas, including the edges of the teapot, the protrusions on the dragon's back, the edges of Benchy, and the tail of the dragon. These results affirm the algorithm's resilience and efficiency in managing detailed and challenging regions under different sampling and noise conditions.

\section{Algorithm}
\label{appendix:pseudo}
To enhance the comprehensibility of our developed algorithm for normal estimation, we also present the pseudo code in Algorithm \ref{alg:point_cloud_normal_estimation}.

\begin{algorithm}[t]
\caption{Pseudo code of the proposed normal estimation algorithm}
\label{alg:point_cloud_normal_estimation}
    \renewcommand{\algorithmicrequire}{\textbf{Input:}}
    \renewcommand{\algorithmicensure}{\textbf{Output:}}
    
\begin{algorithmic}[1]
    \REQUIRE Point Cloud $P$
    \ENSURE Normal $N$ of $P$
    
    \STATE $K \leftarrow \text{InitializeKDTree}(P)$
    \STATE $G \leftarrow \text{ConstructGraph}(P, K)$
    \STATE $C \leftarrow \emptyset$, $W \leftarrow \emptyset$, $N_{temp} \leftarrow \emptyset$, $N_{patch} \leftarrow \emptyset$, $N_{final} \leftarrow \emptyset$
    \WHILE{$C \neq P$}
        \STATE $q \leftarrow \text{SelectUncoveredPoint}(P, C)$
        \IF{$q$ not in independent cluster}
            \STATE $p \leftarrow \text{ConstructPatchDijkstra}(q, G)$
        \ELSE
            \STATE $p \leftarrow \text{ConstructPatchKNN}(q, K)$
        \ENDIF
        \STATE Add points in $p$ close to the center to $C$

        \STATE $N_{temp} \leftarrow \emptyset$
        \FOR{each orientation in Frame}
            \STATE $p_{aligned} \leftarrow \text{FrameTransform}(p)$
            \STATE $n_{out} \leftarrow \text{Network}(p_{aligned})$
            \STATE $n \leftarrow \text{FrameInverseTransform}(n_{out})$           
            \STATE $N_{temp} \leftarrow N_{temp} \cup{n}$
        \ENDFOR
        \STATE Sort $p$ based on distance to $q$
        \STATE $p_{half} \leftarrow$ Nearest half of $p$ to $q$
        \STATE $N_{patch} \leftarrow N_{patch} \cup \{ \text{mean}(N_{temp}) \}$
        \STATE $W \leftarrow W \cup \{ \text{GaussianWeights}(p_{half}, q) \}$
    \ENDWHILE

    \FOR{each $i$ in $P$}
        \STATE $N_i \leftarrow \text{CalculateFinalNormal}(i, N_{patch}, W)$
        \STATE $N_{final} \leftarrow N_{final} \cup \{ N_i \}$
    \ENDFOR

    \RETURN $N_{final}$ 
\end{algorithmic}
\end{algorithm}

\section{Point Cloud Denoising}
\label{appendix:denoise}
To validate the capability of our proposed approach, which combines random frame training with average frame inference, we also apply our framework to point cloud denoising. We utilize the current state-of-the-art algorithm, IterativePFN~\cite{de2023iterativepfn}, as the backbone and conduct training and testing on the benchmark PUNet~\cite{yu2018pu} dataset.
We use Chamfer distance (CD) and the Point2Mesh distance (P2M) as evaluation metrics to measure the denoising effects. All the result metrics are multiplied by \(10^5\).
IterativePFN employs a strategy of random rotation during training for data augmentation. Our improvement is to apply our random frame training strategy during training and utilize the average frame strategy during inference. The results, as shown in Table~\ref{tab:denoise} and Fig.~\ref{fig:denoise}, reveal that our algorithm outperforms IterativePFN in various scenarios, including point clouds of densities 10k and 50k, and noise levels of 1\%, 2\%, and 2.5\%. As demonstrated in Fig.~\ref{fig:denoise}, our algorithm exhibits improvements in both sharp edges and flat surfaces.

Our random frame strategy surpasses conventional random rotation training by eliminating the need for the network to learn from point cloud patches in all poses. Instead, it selectively uses the frame set for training, resulting in higher accuracy and better data utilization. This method is more efficient and effective for point cloud learning compared to previous random rotation approaches.

\begin{figure*}
    \centering
    \includegraphics[width=1\linewidth]{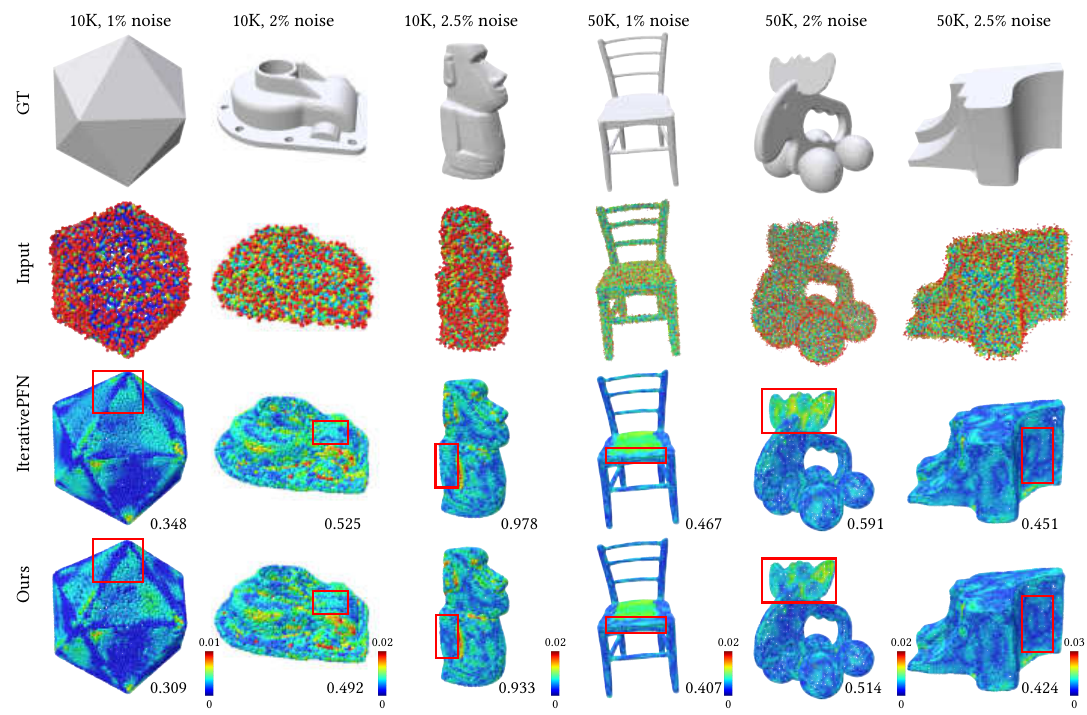}
    %\vskip -0.3cm
    \caption{Qualitative comparisons of point cloud denoising on the PUNet dataset. The values below each model indicate the Point2Mesh distance (P2M). The numbers above each column represent the size of the point cloud and the noise level.}
    \label{fig:denoise}
\end{figure*}

\begin{table*}[t]
  \centering
\caption{Quantitative comparisons of point cloud denoising  on the PUNet dataset. \textbf{Bold} fonts indicate the top performer.}
%\scalebox{0.75}{
\renewcommand{\arraystretch}{1.5} 
\begin{adjustbox}{width=\textwidth}
\begin{tabular}{c|c|c|c|c|c|c|c|c|c|c|c|c}

\hline 
\multirow{3}{*}{Method} & \multicolumn{6}{c|}{$10K$ points} & \multicolumn{6}{c}{$50K$ points} \\ 
\cline{2-13} 
& \multicolumn{2}{c|}{$1\%$ noise} & \multicolumn{2}{c|}{$2\%$ noise} & \multicolumn{2}{c|}{$2.5\%$ noise} & \multicolumn{2}{c|}{$1\%$ noise} & \multicolumn{2}{c|}{$2\%$ noise} & \multicolumn{2}{c}{$2.5\%$ noise} \\
\cline{2-13} 
& CD & P2M & CD & P2M & CD & P2M & CD & P2M & CD & P2M & CD & P2M \\

\hline Noisy & 36.9 & 16.03 & 79.39 & 47.72 & 105.02 & 70.03 & 18.69 & 12.82 & 50.48 & 41.36 & 72.49 & 62.03 \\
\hline PCN~\cite{rakotosaona2020pointcleannet} & 36.86 & 15.99 & 79.26 & 47.59 & 104.86 & 69.87 & 11.03 & 6.46 & 19.78 & 13.7 & 32.03 & 24.86 \\
\hline GPDNet~\cite{pistilli2020gpd} & 23.1 & 7.14 & 42.84 & 18.55 & 58.37 & 30.66 & 10.49 & 6.35 & 32.88 & 25.03 & 50.85 & 41.34 \\
\hline DMRDenoise~\cite{luo2020differentiable} & 47.12 & 21.96 & 50.85 & 25.23 & 52.77 & 26.69 & 12.05 & 7.62 & 14.43 & 9.7 & 16.96 & 11.9 \\
\hline PDFlow~\cite{mao2022pdflow} & 21.26 & 6.74 & 32.46 & 13.24 & 36.27 & 17.02 & 6.51 & 4.16 & 12.7 & 9.21 & 18.74 & 14.26 \\
\hline ScoreDenoise~\cite{luo2021score} & 25.22 & 7.54 & 36.83 & 13.8 & 42.32 & 19.04 & 7.16 & 4.0 & 12.89 & 8.33 & 14.45 & 9.58 \\
\hline Pointfilter~\cite{zhang2020pointfilter} & 24.61 & 7.3 & 35.34 & 11.55 & 40.99 & 15.05 & 7.58 & 4.32 & 9.07 & 5.07 & 10.99 & 6.29 \\
\hline IterativePFN~\cite{de2023iterativepfn} & 20.56 & 5.01 & 30.43 & 8.45 & 33.52 & 10.45 & 6.05 & 3.02 & 8.03 & 4.36 & 10.15 & 5.88 \\
\hline Ours & \textbf{19.87} & \textbf{4.93} &\textbf{29.99} & \textbf{8.18} & \textbf{33.03} & \textbf{10.11} & \textbf{5.92}& \textbf{2.94} & \textbf{7.78} & \textbf{4.22} & \textbf{9.93} & \textbf{5.78}\\
\hline
\end{tabular}
%}
\end{adjustbox}
  \label{tab:denoise}

\end{table*}
% \bibliographystyle{ACM-Reference-Format}

% {\small
% 	% \bibliographystyle{unsrt}
% 	%\bibliographystyle{ieeenat_fullname}
% 	% \bibliography{appendix}
% }	










\newpage
\section*{NeurIPS Paper Checklist}

\begin{enumerate}

\item {\bf Claims}
    \item[] Question: Do the main claims made in the abstract and introduction accurately reflect the paper's contributions and scope?
    \item[] Answer: \answerYes{}
    \item[] Justification: The main claims made in the abstract and introduction, including the efficient E(3)-equivariant method, the Gaussian weight loss function, and receptive-aware inference strategies, accurately reflect the paper's contributions and scope.

\item {\bf Limitations}
    \item[] Question: Does the paper discuss the limitations of the work performed by the authors?
    \item[] Answer: \answerYes{} % Replace by \answerYes{}, \answerNo{}, or \answerNA{}.
    \item[] Justification: The paper discusses limitations in the conclusions section, such as its reliance on synthetic data for training, which could be mitigated by incorporating real-world data and improved simulators.

\item {\bf Theory Assumptions and Proofs}
    \item[] Question: For each theoretical result, does the paper provide the full set of assumptions and a complete (and correct) proof?
    \item[] Answer: \answerYes{} % Replace by \answerYes{}, \answerNo{}, or \answerNA{}.
    \item[] Justification: The theoretical aspects are thoroughly detailed, with all necessary assumptions and proofs provided in the appendix.


    \item {\bf Experimental Result Reproducibility}
    \item[] Question: Does the paper fully disclose all the information needed to reproduce the main experimental results of the paper to the extent that it affects the main claims and/or conclusions of the paper (regardless of whether the code and data are provided or not)?
    \item[] Answer: \answerYes{} % Replace by \answerYes{}, \answerNo{}, or \answerNA{}.
    \item[] Justification: The paper includes detailed descriptions of the datasets, model architectures, training parameters, and evaluation metrics to ensure reproducibility.


\item {\bf Open access to data and code}
    \item[] Question: Does the paper provide open access to the data and code, with sufficient instructions to faithfully reproduce the main experimental results, as described in supplemental material?
    \item[] Answer: \answerYes{} % Replace by \answerYes{}, \answerNo{}, or \answerNA{}.
    \item[] Justification: We provide our code with sufficient instructions to reproduce the main experimental results. Both the training and testing code are included in the supplemental material.


\item {\bf Experimental Setting/Details}
    \item[] Question: Does the paper specify all the training and test details (e.g., data splits, hyperparameters, how they were chosen, type of optimizer, etc.) necessary to understand the results?
    \item[] Answer: \answerYes{} % Replace by \answerYes{}, \answerNo{}, or \answerNA{}.
    \item[] Justification: All relevant details regarding the training and testing setups, including data splits, hyperparameters, and optimizers, are provided in the experimental results section and supplementary material.


\item {\bf Experiment Statistical Significance}
    \item[] Question: Does the paper report error bars suitably and correctly defined or other appropriate information about the statistical significance of the experiments?
    \item[] Answer: \answerNo{} % Replace by \answerYes{}, \answerNo{}, or \answerNA{}.
    \item[] Justification: The paper does not include error bars or statistical significance tests for the experimental results, focusing instead on overall performance metrics.


\item {\bf Experiments Compute Resources}
    \item[] Question: For each experiment, does the paper provide sufficient information on the computer resources (type of compute workers, memory, time of execution) needed to reproduce the experiments?
    \item[] Answer: \answerYes{} % Replace by \answerYes{}, \answerNo{}, or \answerNA{}.
    \item[] Justification: The paper specifies the use of NVIDIA 4090 GPUs, detailing the training details and batch sizes, thereby providing sufficient information on the compute resources used.

    
\item {\bf Code Of Ethics}
    \item[] Question: Does the research conducted in the paper conform, in every respect, with the NeurIPS Code of Ethics \url{https://neurips.cc/public/EthicsGuidelines}?
    \item[] Answer: \answerYes{} % Replace by \answerYes{}, \answerNo{}, or \answerNA{}.
    \item[] Justification:  The research aligns with the NeurIPS Code of Ethics, ensuring transparency, reproducibility, and consideration of ethical implications.

\item {\bf Broader Impacts}
    \item[] Question: Does the paper discuss both potential positive societal impacts and negative societal impacts of the work performed?
    \item[] Answer:  \answerYes{} % Replace by \answerYes{}, \answerNo{}, or \answerNA{}.
    \item[] Justification: We have presented a paragraph to discuss the broader impacts of our work. 
    The paper focuses on a fundamental geometry process task, and thus there is no direct negative societal impacts of the work performed.
    
\item {\bf Safeguards}
    \item[] Question: Does the paper describe safeguards that have been put in place for responsible release of data or models that have a high risk for misuse (e.g., pretrained language models, image generators, or scraped datasets)?
    \item[] Answer: \answerNA{} % Replace by \answerYes{}, \answerNo{}, or \answerNA{}.
    \item[] Justification:  The paper does not involve the release of high-risk data or models that would require special safeguards.


\item {\bf Licenses for existing assets}
    \item[] Question: Are the creators or original owners of assets (e.g., code, data, models), used in the paper, properly credited and are the license and terms of use explicitly mentioned and properly respected?
    \item[] Answer: \answerYes{} % Replace by \answerYes{}, \answerNo{}, or \answerNA{}.
    \item[] Justification: The paper credits all datasets and models used, including proper citations and adherence to license terms where applicable.

\item {\bf New Assets}
    \item[] Question: Are new assets introduced in the paper well documented and is the documentation provided alongside the assets?
    \item[] Answer: \answerNA{} % Replace by \answerYes{}, \answerNo{}, or \answerNA{}.
    \item[] Justification: The paper does not introduce new assets, focusing on the development and evaluation of the proposed method.

\item {\bf Crowdsourcing and Research with Human Subjects}
    \item[] Question: For crowdsourcing experiments and research with human subjects, does the paper include the full text of instructions given to participants and screenshots, if applicable, as well as details about compensation (if any)? 
    \item[] Answer: \answerNA{} % Replace by \answerYes{}, \answerNo{}, or \answerNA{}.
    \item[] Justification: The paper does not involve crowdsourcing or research with human subjects.


\item {\bf Institutional Review Board (IRB) Approvals or Equivalent for Research with Human Subjects}
    \item[] Question: Does the paper describe potential risks incurred by study participants, whether such risks were disclosed to the subjects, and whether Institutional Review Board (IRB) approvals (or an equivalent approval/review based on the requirements of your country or institution) were obtained?
    \item[] Answer: \answerNA{} % Replace by \answerYes{}, \answerNo{}, or \answerNA{}.
    \item[] Justification: The paper does not involve research with human subjects, thus no IRB approval is required.

\end{enumerate}

\end{document}